\documentclass{article}
\usepackage{amsmath}
\usepackage{amsfonts}
\usepackage{amssymb}
\usepackage{amsthm}
\usepackage{xcolor}
\usepackage{microtype,amssymb,amsmath}
\usepackage{graphicx}
\usepackage{subfigure}
\usepackage{booktabs} 
\usepackage{hyperref}
\usepackage{algorithmic}

\usepackage[accepted]{icml2021}
\icmltitlerunning{Taylor Expansions of Discount Factors}

\usepackage{wrapfig}

\usepackage{hyperref}
\hypersetup{
    colorlinks=true,
    linkcolor=blue,
    citecolor=cyan,
}
\usepackage{natbib}
\usepackage{xcolor}
\usepackage{amsthm,amssymb,amsmath}
\usepackage{mathrsfs}
\usepackage{mathtools}

\usepackage{enumitem}
\usepackage{natbib}
\usepackage{verbatim}
\usepackage{url}
\usepackage{thm-restate}

\usepackage{bbm}

\usepackage{parskip}

\usepackage{graphbox}

\usepackage{subfigure}

\newtheoremstyle{definition}
{3pt} 
{3pt} 
{} 
{} 
{\bfseries} 
{.} 
{.5em} 
{} 

\theoremstyle{definition}

\renewcommand{\epsilon}{\varepsilon}
\renewcommand{\hat}{\widehat}
\renewcommand{\tilde}{\widetilde}
\renewcommand{\bar}{\overline}

\usepackage{selectp}

\usepackage{verbatim}

\begin{document}

\twocolumn[
\icmltitle{Taylor Expansions of Discount Factors}

\icmlsetsymbol{equal}{*}

\begin{icmlauthorlist}
\icmlauthor{Yunhao Tang}{cu}
\icmlauthor{Mark Rowland}{dml}
\icmlauthor{R\'emi Munos}{dmp}
\icmlauthor{Michal Valko}{dmp}
\end{icmlauthorlist} 

\icmlaffiliation{cu}{Columbia University, New York, USA}
\icmlaffiliation{dmp}{DeepMind, Paris, France}
\icmlaffiliation{dml}{DeepMind, London, UK}
\icmlcorrespondingauthor{yt2541@columbia.edu}{Yunhao}

\icmlkeywords{Machine Learning, ICML}

\vskip 0.3in
]



\printAffiliationsAndNotice{}  

\begin{abstract}
  In practical reinforcement learning (RL), the discount factor used for estimating value functions often differs from that used for defining the evaluation objective. In this work, we study the effect that this discrepancy of discount factors has during learning, and discover a family of objectives that interpolate value functions of two distinct discount factors. Our analysis suggests new ways for estimating value functions and performing policy optimization updates, which demonstrate empirical performance gains. This framework also leads to new insights on commonly-used deep RL heuristic modifications to policy optimization algorithms.
\end{abstract}

\section{Introduction}
\label{sec:intro}

One of the most popular models for reinforcement learning (RL) is the Markov decision process (MDP) with exponential discounting over an infinite horizon \citep{sutton1998introduction,puterman2014markov}, with discounted objectives of the following form
\begin{align*}
    V_\gamma^\pi(x)=   \mathbb{E}_{\pi}\left\lbrack \sum_{t=0}^{\infty} \gamma^t r_t \; \middle| \; x_0=x \right\rbrack \, .
\end{align*}
Discounted models enjoy favorable theoretical properties, and are also the foundation of many practical RL algorithms that enjoy empirical success (e.g. see \citep{mnih2015human,schulman2015trust,lillicrap2015continuous,schulman2017proximal}). However, in most applications of RL, the objective of interest is the expected \emph{undiscounted cumulative return},
\begin{align}
    \mathbb{E}_{\pi}\left\lbrack \sum_{t=0}^T r_t \; \middle| \; x_0=x \right\rbrack,\label{eq:rl-obj}
\end{align} 
where $T<\infty$ is a (possibly random) evaluation horizon, which usually also denotes the end of the trajectory. For example, $T$ could be the first time the MDP gets into a terminal state (e.g., a robot falls); when the MDP does not have a natural terminal state, $T$ could be enforced as a deterministic horizon. This creates a technical gap between algorithmic developments and implementations: it is tempting to design algorithms that optimize $V_\gamma^\pi(x)$, however, further  heuristics are often needed to get strong practical performance. This issue manifests itself with the policy gradient (PG) theorem \citep{sutton2000policy}. Let $\pi_\theta$ be a parameterized policy. The policy gradient (PG) $\nabla_\theta V_\gamma^{\pi_\theta}(x)$ is computed as
\begin{align}
       \mathbb{E}_{\pi_\theta}\left\lbrack \sum_{t=0}^{\infty} \gamma^t Q_\gamma^{\pi_\theta}(x_t,a_t)\nabla_\theta \log\pi_\theta(a_t|x_t) \; \middle| \; x_0=x \right\rbrack\,.
      \label{eq:pg}
\end{align}
However, the practical implementation of PG updates usually omits the discount factors (see for example the high-quality open source packages \citep{dhariwal2017openai,SpinningUp2018}, leading to an approximate gradient of the form
\begin{align}
      \mathbb{E}_{\pi_\theta}\left\lbrack \sum_{t=0}^{T}  Q_\gamma^{\pi_\theta}(x_t,a_t)\nabla_\theta \log\pi_\theta(a_t|x_t)\; \middle| \; x_0=x \right\rbrack \, .
      \label{eq:uniform-gradient}
\end{align}
Most prior work on PG algorithms rely on this heuristic update to work properly in deep RL applications. The intuitive argument for dropping the factor $\gamma^t$ is that Eqn~\eqref{eq:pg} optimizes $V_\gamma^{\pi_\theta}(x)$, which is very myopic compared to the objective in Eqn~\eqref{eq:rl-obj}. Consequently, the exponential discount $\gamma^t$ is too aggressive for weighting updates with large $t$. As a concrete example, in many MuJoCo control tasks \citep{brockman2016openai}, the most commonly used discount factor is $\gamma=0.99$. This leads to an effective horizon of $\frac{1}{1-\gamma}=100$, which is much smaller than the evaluation horizon $T=1000$. This technical gap between theory and practice has been alluded to previously (by e.g., \citealp{o2016combining}) and is explicitly discussed by \citet{nota2019policy}.

To bypass this gap, a straightforward solution would be to na\"ively increase the discount factor $\gamma\geq1-\frac{1}{T}$ and apply the PG in Eqn~\eqref{eq:pg}. In the example above, this implies using $\gamma\geq0.999$. Unfortunately, this rarely works well in practice, as we will also see in experiments. The failure might be due to the higher variance of the estimation \citep{schulman2015high} or the collapse of the action gaps \citep{lehnert2018value,laroche2018reinforcement}, which is aggravated when combined with function approximations.

Nevertheless, as a theoretical framework, it is insightful to emulate the undiscounted objective in Eqn~\eqref{eq:rl-obj}  using the (un)discounted objective $V_{\gamma^\prime}^\pi(x)$ with $\gamma^\prime\geq 1-\frac{1}{T}$.  To build intuitions about this approximation, note that when the time step is small $t\ll T$, the multiplicative factor $(\gamma^\prime)^t \approx 1$ and the cumulative rewards are almost undiscounted; even when $t=T$, we have $(\gamma^\prime)^t\geq (1-\frac{1}{T})^T \approx \frac{1}{e} \gg 0$. Overall, this is a much more accurate approximation than $V_\gamma^\pi(x)$. This naturally prompts us to answer the following general question: 
\emph{How do we evaluate and optimize $V_{\gamma^\prime}^\pi(x)$ with estimates built for $V_\gamma^\pi(x)$ where $0<\gamma<\gamma^\prime \leq 1$?}

\paragraph{Main idea.} 
We study the relation between $V_\gamma^\pi(x)$ and $V_{\gamma^\prime}^\pi(x)$ via Taylor expansions. In Section~\ref{sec:taylor}, we identify a family of interpolating objectives between the more myopic objective $V_\gamma^\pi(x)$ and the true objective of interest $V_{\gamma^\prime}^\pi(x)$. In Section~\ref{sec:opt}, 
we start with insights on why the heuristic in Eqn~\eqref{eq:uniform-gradient} might be useful in practice. Then, we apply Taylor expansions directly to the heuristic updates, to arrive at a family of interpolating updates. In Section~\ref{sec:algo}, we build on theoretical insights to derive improvements to established deep RL algorithms. We show their performance gains in Section~\ref{sec:exp}.

\section{Background}

Consider the setup of a MDP. At any discrete time $t\geq 0$, the agent is in state $x_t\in \mathcal{X}$, takes an action $a_t\in \mathcal{A}$, receives an instant reward $r_t=r(x_t,a_t)\in [0, R_{\text{max}}]$ and transitions to a next state $x_{t+1}\sim p(\cdot|x_t,a_t)$. For simplicity, we assume $r(x,a)$ to be deterministic. Let policy $\pi:\mathcal{X} \rightarrow \mathcal{P}(\mathcal{A})$ be a mapping from states to distributions over actions. Let $\gamma\in [0,1)$ be a discount factor, define the Q-function $Q_\gamma^\pi(x,a) \coloneqq \mathbb{E}_\pi\left\lbrack\sum_{t=0}^\infty \gamma^t r_t \; \middle| \; x_0=x, a_0=a \right\rbrack$  and value function $V_\gamma^\pi(x) \coloneqq \mathbb{E}_\pi\left\lbrack\sum_{t=0}^\infty \gamma^t r_t \; \middle| \; x_0=x \right\rbrack$. We also define the advantage function $A_\gamma^\pi(x,a)\coloneqq Q_\gamma^\pi(x,a)-V_\gamma^\pi(x)$. Here, $\mathbb{E}_\pi \left\lbrack \cdot\right\rbrack$ denotes that the trajectories $(x_t,a_t,r_t)_{t=0}^\infty$ are generated under policy $\pi$. Throughout the paper, we use subscripts $\gamma$ to emphasize that RL quantities implicitly depend on discount factors.

\subsection{Linear programs for reinforcement learning}

Henceforth, we assume all vectors to be column vectors. The value functions $V_\gamma^\pi$ satisfy the Bellman equations $V_\gamma^\pi(x)=\mathbb{E}_{\pi}\left\lbrack r(x,a) +\gamma V_\gamma^\pi (x^\prime) \mid x_0 = x \right\rbrack$ \citep{bellman1957markovian}. Such equations can be encoded into a linear program (LP) \citep{de2003linear,puterman2014markov}. Let $V\in\mathbb{R}^{\mathcal{X}}$ be the primal variables, consider the following LP,
\begin{align}
            \max \  \delta_x^T V, \ \ 
            V=r^\pi + \gamma P^\pi V,
            \label{eq:primal-lp}
\end{align}

where  $r^\pi\in\mathbb{R}^{\mathcal{X}}$ is the state-dependent reward $r^\pi(x^\prime)\coloneqq \sum_{a^\prime}\pi(a^\prime|x^\prime)r(x^\prime,a^\prime) $ and $P^\pi\in\mathbb{R}^{\mathcal{X} \times \mathcal{X}}$ is the transition matrix under $\pi$. Here, $\delta_x\in\mathbb{R}^{\mathcal{X}}$ encodes the one-hot distribution (Dirac) at $x$. Similar results hold for considering the LP objective $v^T V$ with a general distribution $v\in\mathcal{P}(\mathcal{X})$. It then follows that the optimal solution to the above LP is $V^\ast=V_\gamma^\pi$. Now, consider the dual LP to Eqn~\eqref{eq:primal-lp}, let $d\in\mathbb{R}^{\mathcal{X}}$ be the dual variables,
\begin{align}
            \min \  (1-\gamma)^{-1} (r^\pi)^T d, \ \ 
            d=(1-\gamma)\delta_x + \gamma (P^\pi)^T d.
            \label{eq:dual-lp}
\end{align}

The optimal solution to the dual program has a natural probabilistic interpretation. It is the discounted visitation distribution $d_{x,\gamma}^\pi$ under policy $\pi$ with starting state $x$ as $d_{x,\gamma}^\pi(x^\prime)\coloneqq (1-\gamma) \sum_{t\geq 0} \gamma^t P_\pi(x_t=x^\prime|x_0=x)$ where $P_\pi(x_t=x^\prime|x_0=x)$ is a probability measure induced by the policy $\pi$ and the MDP transition kernel. By strong duality, the value function can be equivalently written as
\begin{align}
    V_\gamma^\pi(x) = \frac{1}{1-\gamma}\mathbb{E}_{x^\prime\sim d_{x,\gamma}^\pi,a^\prime\sim\pi(\cdot|x^\prime)}\left\lbrack r(x^\prime,a^\prime) \right\rbrack.
    \label{eq:visitation-value}
\end{align}  

\section{Taylor Expansions of Value Functions}
\label{sec:taylor}

Below, we show how to estimate $V_{\gamma^\prime}^\pi(x)$ with approximations constructed from value functions $V_\gamma^\pi(x)$ for $\gamma<\gamma'$. Unless otherwise stated, we always assume $\gamma'<1$ for a more convenient mathematical treatment of the problem.

\subsection{Taylor expansions of discount factors}
\label{sec:primal}
We start with some notations: we abuse the notation of value functions $V_\gamma^\pi \in \mathbb{R}^{\mathcal{X}}$ to both refer to the scalar function as well as a vector. The Bellman equation for the value-function is expressed in the matrix form \citep{puterman2014markov}
\begin{align}
    V_{\gamma^\prime}^\pi = r^\pi + \gamma^\prime P^\pi V_{\gamma^\prime}^\pi.
    \label{eq:primal-bellman}
\end{align}
Inverting the equation, 
\begin{align}
    V_{\gamma^\prime}^\pi = (I-\gamma^\prime P^\pi)^{-1} r^\pi.\label{eq:primal-start-expansion}
\end{align}
Now, we present the main result of Taylor expansions.
\begin{restatable}{proposition}{propstaylor}\label{prop:taylor-expansion}
The following holds for all $K\geq0$,
\begin{align}
    V_{\gamma^\prime}^\pi &= \sum_{k=0}^K \left((\gamma^\prime-\gamma) (I-\gamma P^\pi)^{-1} P^\pi \right)^k V_\gamma^\pi \nonumber \\
    &+ \underbrace{\left((\gamma^\prime-\gamma) (I-\gamma P^\pi)^{-1} P^\pi \right)^{K+1} V_{\gamma^\prime}^\pi}_{\text{residual}}. \label{eq:fundamental-taylor}
\end{align}
When $\gamma<\gamma^\prime<1$, the residual norm converges to $0$, which implies 
\begin{align}
    V_{\gamma^\prime}^\pi = \sum_{k= 0}^\infty \left((\gamma^\prime-\gamma) (I-\gamma P^\pi)^{-1} P^\pi \right)^k V_\gamma^\pi. \label{eq:primal-expansion}
\end{align}
\end{restatable}
We provide a proof sketch here: Note that $\gamma^\prime P^\pi = (\gamma^\prime-\gamma)P^\pi + \gamma P^\pi$ and apply the Woodbury matrix identity to obtain $(I-{\gamma^\prime}P^\pi)^{-1}=(I-\gamma P^\pi)^{-1}+(\gamma^\prime-\gamma)(I-\gamma P^\pi)^{-1}P^\pi (I-\gamma^\prime P^\pi)^{-1}$. We can then recursively expand Eqn~\eqref{eq:primal-start-expansion} $K$ times to arrive at Eqn~\eqref{eq:fundamental-taylor}. In particular, by expanding the equation once, we see that $(I-\gamma' P^\pi)^{-1}$ is equivalent to the following,
\begin{align*}
    & (I-\gamma P^\pi)^{-1} +  (\gamma'-\gamma)(I-\gamma P^\pi)^{-1} P^\pi (I-\gamma P^\pi)^{-1}  \\
    &+ (\gamma'-\gamma)^2 \left((I-\gamma P^\pi)^{-1} P^\pi \right)^2 \underbrace{ (I-\gamma' P^\pi)^{-1} }_{\text{can be expanded further}},
\end{align*}
where the last term can be expanded further by plugging in the Woodbury matrix identity. See the complete proof in Appendix~\ref{appendix:proof}.

\paragraph{Extensions to $\gamma'=1$.} The above result can extend to the case $\gamma'=1$.  We make two assumptions: \textbf{A.1} The Markov chain induced by $\pi$ is absorbing and $T$ is the absorption time; \textbf{A.2}  $r^\pi(x)=0$ for absorbing states $x$. Under these assumptions, we can interpret such absorbing states as the terminal states. As a result,  $V_{\gamma'=1}^\pi(x)=\mathbb{E}_\pi\left\lbrack \sum_{t=0}^{T} r_t \; \middle| \; x_0=x \right\rbrack$ is well-defined and  Proposition~\ref{prop:taylor-expansion} still holds; see Appendix~\ref{appendix:proof} for the complete proof. 

In practice, it is infeasible to sum up all infinite number of terms in the Taylor expansion. It is then of interest to consider the $K$\textsuperscript{th}-order expansion of $V_{\gamma^\prime}^\pi$, which truncates the infinite series. Specifically, we define the $K$\textsuperscript{th}-order expansion as 
\begin{align}
    V_{K,\gamma,\gamma^\prime}^\pi \coloneqq  \sum_{k=0}^K ((\gamma^\prime-\gamma) (I-\gamma P^\pi)^{-1} P^\pi )^k V_\gamma^\pi \, .
    \label{eq:kth-order-primal}
\end{align}

As $K$ increases, the $K$\textsuperscript{th} order expansion becomes increasingly close to the infinite series, which evaluates to $V_{\gamma^\prime}^\pi(x)$. This is formalized next. 
\begin{restatable}{proposition}{properrork}\label{prop:kth-expansion}
The following bound holds for all $K\geq0$,
\begin{align}
    \left|V_{\gamma^\prime}^\pi(x) - V_{K,\gamma,\gamma^\prime}^\pi(x) \right| \leq \left(\frac{\gamma^\prime-\gamma}{1-\gamma}\right)^{K+1}\frac{R_\text{max}}{1-\gamma^\prime}.
\end{align}
\end{restatable}

\subsection{Sample-based approximations of Taylor expansions}

We now describe how to estimate $V_{K,\gamma,\gamma^\prime}^\pi(x)$ via samples. First, we build some intuition on the behavior of expansions at different orders $K$ by considering a few special cases.

\paragraph{Zeroth-order expansion.} By setting $K=0$, we see that 
\begin{align}
    V_{0,\gamma,\gamma^\prime}^\pi = V_\gamma^\pi. \label{eq:zeroth-order-primal}
\end{align}
The zeroth order expansion approximates the value function $V_{\gamma^\prime}^\pi(x)$ of the discount factor $\gamma^\prime$ with that $V_\gamma^\pi(x)$ of a lower discount factor $\gamma<\gamma^\prime$. This is a very straightforward approximation to use in that no sampling at all is required, but it may not be accurate.

\paragraph{First-order expansion.} When $K=1$, we consider the increments of the expansions,
\begin{align}
    V_{1,\gamma,\gamma^\prime}^\pi - V_{0,\gamma,\gamma^\prime}^\pi = (\gamma^\prime-\gamma) (I-\gamma P^\pi)^{-1} P^\pi V_\gamma^\pi. \label{eq:primal-first-order}
\end{align}

To understand the first order expansion, recall that in the definition of value function $V_\gamma^\pi=(I-\gamma P^\pi)^{-1}r^\pi$, immediate rewards $r^\pi$ are \emph{accumulated} via the matrix $(I-\gamma P^\pi)^{-1}$. In general, for any $X,Y\in\mathbb{R}^{\mathcal{X}}$, we can interpret $X=(I-\gamma P^\pi))^{-1}Y$ as accumulating $Y$ as rewards to compute $X$ as value functions. By analogy, we can interpret the RHS of Eqn~\eqref{eq:primal-first-order} as the value function assuming $(\gamma'-\gamma)P^\pi V_\gamma^\pi$ as immediate rewards. In other words, the first order expansion bootstraps the zeroth order expansion $V_\gamma^\pi$ to form a more accurate approximation. Combined with the zeroth order expansion, we can also conveniently write the difference of first- and zeroth-order expansions as an expectation $V_{1,\gamma,\gamma'}^\pi(x) - V_{0,\gamma,\gamma'}(x) = (\gamma'-\gamma) \mathbb{E}_{\pi}\left\lbrack \sum_{t=1}^\infty \gamma^{t-1} V_\gamma^\pi(x_{t}) \; \middle| \; x_0=x \right\rbrack$. Let $\tau\sim\text{Geometric}(1-\gamma)$ be a random time such that $P(\tau=t)=(1-\gamma)\gamma^{t},\forall t\in\mathbb{Z}_{\geq 1}$. The difference can also be expressed via this random time
\begin{align*}
    V_{1,\gamma,\gamma'}^\pi(x) - V_{0,\gamma,\gamma'}(x)  = \frac{\gamma'-\gamma}{1-\gamma} \mathbb{E}_{\pi,\tau}\left\lbrack V_\gamma^\pi(x_{\tau}) \right\rbrack.
\end{align*}
Note that from this expression, we obtain a simple unbiased estimate for $V_{1,\gamma,\gamma'}^\pi(x) - V_{0,\gamma,\gamma'}(x)$, using a sampled trajectory and a random time step $\tau$.

\paragraph{General $K$\textsuperscript{th}-order expansion.}

We now present results for general $K$. Consider the incremental term,
\begin{align}
    V_{K,\gamma,\gamma^\prime}^\pi - V_{K-1,\gamma,\gamma^\prime}^\pi = (\gamma^\prime-\gamma)^K \left( (I-\gamma P^\pi)^{-1} P^\pi \right)^K V_\gamma^\pi. \label{eq:primal-k-order}
\end{align}
Note that the aggregate matrix $\left( (I-\gamma P^\pi)^{-1} P^\pi \right)^K$ suggests a recursive procedure to bootstrap from lower order expansions to construct higher order expansions. To see why, we can rewrite the right-hand side of Eqn~\eqref{eq:primal-k-order} as
\begin{align*}
     (\gamma'-\gamma) (I-\gamma P^\pi)^{-1} P^\pi \left( V_{K-1,\gamma,\gamma^\prime}^\pi - V_{K-2,\gamma,\gamma^\prime}^\pi \right). 
\end{align*}
Indeed, we can interpret the difference $ V_{K,\gamma,\gamma^\prime}^\pi - V_{K-1,\gamma,\gamma^\prime}^\pi$ as the value function under the immediate reward  $(\gamma'-\gamma) P^\pi\left( V_{K-1,\gamma,\gamma^\prime}^\pi - V_{K-2,\gamma,\gamma^\prime}^\pi\right)$. This generalizes the bootstrap procedure of the first order expansion as a special case where we naturally assume $V_{-1,\gamma,\gamma'}^\pi=0$. Given $K$ i.i.d.\,random times $\tau_i\sim \text{Geometric}(1-\gamma)$, we can write $V_{K,\gamma,\gamma^\prime}^\pi(x) - V_{K-1,\gamma,\gamma^\prime}^\pi(x)$ as the expectation
\begin{align*}
    \left(\frac{\gamma'-\gamma}{1-\gamma}\right)^K \mathbb{E}_{\tau_i,1\leq i\leq K}\left\lbrack V_\gamma^\pi\left(x_{\tau_1+\dots+\tau_K}\right) \right\rbrack\,.
\end{align*} 
Based on the above expression, Algorithm~\ref{algo:estimate} provides a subroutine that generates unbiased estimates of $V_{K,\gamma,\gamma^\prime}^\pi(x)$ by sub-sampling an infinite trajectory $(x_t,a_t,r_t)_{t=0}^\infty$ with the random times.

\paragraph{Practical implementations.} While the above and Algorithm 1 show how to compute one-sample estimates, in practice, we might want to average multiple samples along a single trajectory for variance reduction. See Appendix~\ref{appendix:exp} for further details on the practical estimates.

\begin{algorithm}[h]
\begin{algorithmic}
\REQUIRE A trajectory $(x_t,a_t,r_t)_{t=0}^\infty\sim\pi$ and discount factors $\gamma<\gamma^\prime<1$\\
\STATE 1. Compute an unbiased estimate $\hat{V}_\gamma^\pi(x_t)$ for states along the trajectory, e.g., $\hat{V}_\gamma^\pi(x_t)=\sum_{t^\prime\geq t} \gamma^{t^\prime-t} r_{t^\prime}$.
\STATE 2. Sample $K$ random time $\{\tau_i\}_{1\leq i\leq K}$, all \vfill i.i.d.\,geometrically distributed $\tau_i\sim \text{Geometric}(1-\gamma)$.
\STATE 3. Return the unbiased estimate $ \sum_{k=0}^K \left(\frac{\gamma'-\gamma}{1-\gamma}\right)^k  \hat{V}_\gamma^\pi(x_{t_k})$ where $t_k=\sum_{i=1}^k \tau_i$.
\caption{Estimating the $K$\textsuperscript{th} order expansion}
\label{algo:estimate}
\end{algorithmic}
\end{algorithm}

\paragraph{Interpretation of expansions in the dual space.}
Recall that $V_{\gamma^\prime}^\pi = (I-\gamma^\prime P^\pi)^{-1}r^\pi = I (I-\gamma^\prime P^\pi)^{-1}r^\pi $ where the identity matrix $I=[\delta_0,\delta_1,...\delta_{\mathcal{X}}]$ concatenates Dirac delta vectors $\delta_x,\forall x\in\mathcal{X}$. Since $r^\pi$ is a constant vector, Taylor expansions essentially construct approximations to the matrix $(I-\gamma^\prime P^\pi)^{-1}$. By grouping the matrix with the reward vector (or the density matrix), we arrive at the primal expansion (or the dual expansion),
\begin{align*}
     I\underbrace{ (I-\gamma^\prime P^\pi)^{-1}r^\pi}_{\text{primal\ expansions\ of}\ V_{\gamma'}^\pi(x)} = \underbrace{I (I-\gamma^\prime P^\pi)^{-1}}_{\text{dual\ expansions\ of}\ d_{x,\gamma'}^\pi}r^\pi
\end{align*}
The derivations above focus on the primal expansion view. We show a parallel theory of dual expansion in Appendix~\ref{appendix:dual}. The equivalence of primal-dual view of Taylor expansions suggests connections with seemingly disparate lines of prior work: \citet{janner2020gamma} propose a density model for visitation distribution of different $\gamma$ in the context of model-based RL. They show that predictions of large discount factors could be bootstrapped from predictions of small discount factors. This corresponds exactly to the dual space expansions, which is equivalent to the primal space expansions.

\paragraph{Extensions to Q-functions.} In Appendix~\ref{appendix:gamma}, we show that it is possible to build approximations to $Q_{\gamma'}^\pi$ using $Q_\gamma^\pi$ as building blocks. The theoretical guarantees and estimation procedures are similar to the case of value functions.

\subsection{Approximation errors with finite samples}
Proposition~\ref{prop:kth-expansion} shows that the \emph{expected} approximation error decays as   $\left| V_{K,\gamma,\gamma^\prime}^\pi(x)-V_{\gamma'}^\pi(x)\right|=O\left(\left(\frac{\gamma^\prime-\gamma}{1-\gamma}\right)^{K+1} \right)$ for $\gamma<\gamma'<1$. This motivates using a high value of $K$ when constructing the approximation. However, in practice, all constituent terms in the $K$\textsuperscript{th} order expansion are random estimates, each with a non-zero variance. This might lead the variance of the overall estimate to increase as $K$ increases. As a result, $K$ mediates a trade-off between bias (expected approximation error) and variance. We formalize such intuitions in Appendix~\ref{appendix:bound}, where we theoretically analyze the trade-off using the phased TD-learning framework \citep{kearns2000bias}.

\paragraph{A numerical example.}
To get direct intuition about the effect of $K$, we focus on a tabular MDP example. The MDP has $|\mathcal{X}|=10$ states and $|\mathcal{A}|=2$ actions. All entries of the transition table $p(y|x,a)$ are generated from a Dirichlet distribution with parameters $(\alpha,\ldots,\alpha)$ with $\alpha=0.01$. The policy $\pi(a|x)$ is uniformly random. We take $\gamma=0.2$ and $\gamma'=0.8$. The agent generates $N=10$ trajectories $(x_t,a_t,r_t)_{t=0}^T$ with a very large horizon $T$ with a fixed starting state $x_0$. We assume access to base estimates $\hat{V}_\gamma^\pi(x_t)$ and the Taylor expansion estimates $\hat{V}_{K,\gamma,\gamma'}^\pi(x_0)$ are computed based on Algorithm~\ref{algo:estimate}. We estimate the relative error as $\hat{E}_K(x_0)=\left| V_{\gamma'}^\pi(x_0) - \hat{V}_{K,\gamma,\gamma'}^\pi(x_0)\right|$. For further experiment details, see Appendix~\ref{appendix:exp}. 

In Figure~\ref{fig:errors}(a), we show how errors vary as a function of $K$. We study two settings: \textbf{(1)} Expected estimates (red), where $\hat{V}_{K,\gamma,\gamma'}^\pi(x_0)$ is computed analytically through access to transition tables. In this case, similar to how the theory suggests, the error decays exponentially; \textbf{(2)} Sample-based estimates (blue) with base estimates $\hat{V}_\gamma^\pi(x_t)=\sum_{s=0}^\infty \gamma^s r_{t+s}$. The errors decay initially with $K$ but later start to increase a bit as $K$ gets large. The optimal $K$ in the middle achieves the best bias-variance trade-off. Note that in this particular example, the estimates do not pay a very big price in variance for large $K$. We speculate this is because increments to the estimates are proportional to $\left(\frac{\gamma'-\gamma}{1-\gamma}\right)^{K+1}$, which scales down additional variance terms quickly as $K$ increases.

In Figure~\ref{fig:errors}(b), we study how the optimal expansion order $K^\ast$ depends on the noise level of base estimates. To emulate the noise, we assume access to base  estimates $\hat{V}_\gamma^\pi(x_t)=V_\gamma^\pi(x_t)+\mathcal{N}(0,\sigma^2)$ for some noise level $\sigma$. The optimal order $K^\ast$ is computed as $K^\ast=\arg\min_k \hat{E}_k(x_0)$. In general, we observe that when $\sigma$ increases, $K^\ast$ decreases. Intuitively, this implies that as the base estimates $\hat{V}_\gamma^\pi(x)$ become noisy, we should prefer smaller value of $K$ to control the variance. This result bears some insights for practical applications such as downstream policy optimization, where we need to select an optimal $K$ for the tasks at hand.

\begin{figure}
    \centering
     \subfigure[Trade-off of $K$]{\includegraphics[keepaspectratio,width=.23\textwidth]{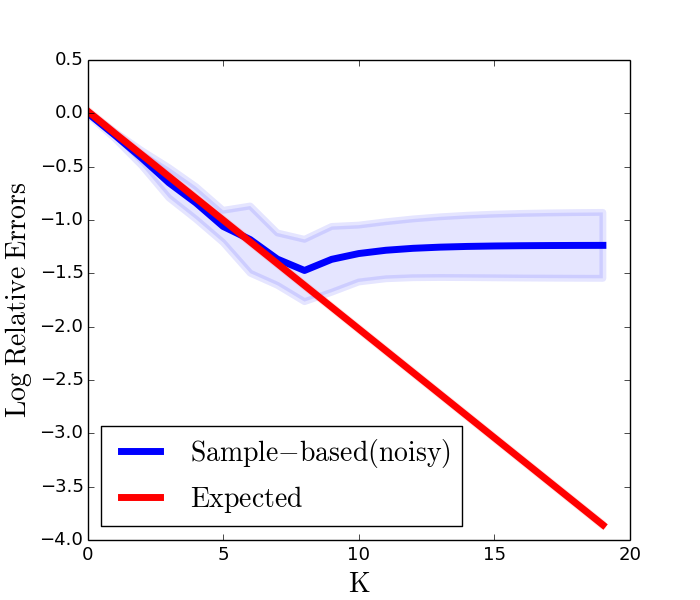}}
     \subfigure[Optimal $K$]{\includegraphics[keepaspectratio,width=.23\textwidth]{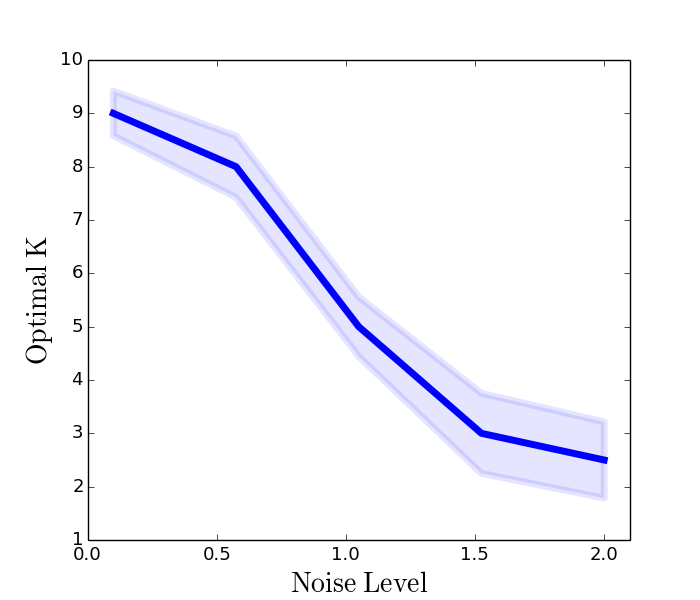}}
    \caption{Comparison of Taylor expansions with different orders. The x-axis shows the order $K$, the y-axis shows the log relative errors of the approximations. The blue curve shows the exact computations while the red curve shows the sample based estimations. See Appendix~\ref{appendix:exp} for more details.}
    \label{fig:errors}
\end{figure}

\section{Taylor Expansions of Gradient Updates}
\label{sec:opt}

In Section~\ref{sec:taylor}, we discussed how to construct approximations to $V_{\gamma'}^\pi(x)$. For the purpose of policy optimization, it is of direct interest to study approximations to $ \nabla_\theta V_{\gamma^\prime}^{\pi_\theta}(x)$. As stated in Section~\ref{sec:intro}, a major premise of our work is that in many practical contexts, estimating discounted values under $\gamma'\approx 1$ is difficult. As a result, directly evaluating the full gradient  $ \nabla_\theta V_{\gamma^\prime}^{\pi_\theta}(x)$ is challenging, because it requires estimating Q-functions $Q_{\gamma'}^{\pi_\theta}(x,a)$. Below, we start by showing how the decomposition of $\nabla_\theta V_{\gamma^\prime}^{\pi_\theta}(x)$ motivates a particular form of gradient update, which is generally considered a deep RL heuristic. Then we construct approximations to this update based on Taylor expansions.

\subsection{$V_{\gamma'}^\pi$ as a weighted mixture of $V_\gamma^\pi$}

We can explicitly rewrite $V_{\gamma'}^\pi(x)$ as a weighted mixture of value functions $V_\gamma^\pi(x'),x'\in\mathcal{X}$. This result was alluded to in \citep{romoff2019separating} and formally shown below.
\begin{restatable}{lemma}{lemmamixture}\label{lemma:value-mixture}
Assume $\gamma<\gamma'<1$. We can write
$
    V_{\gamma'}^\pi(x) = (\rho_{x,\gamma,\gamma'}^\pi)^T V_\gamma^\pi
$, where the weight vector $\rho_{x,\gamma,\gamma'}^\pi\in\mathbb{R}^{\mathcal{X}}$ is
\begin{align*}
    \left(I-\gamma (P^\pi)^T\right) \left(I-\gamma' (P^\pi)^T\right)^{-1}\delta_x.
\end{align*} 
Also we can rewrite $V_{\gamma'}^\pi(x)$,  using an expectation, as:
\begin{align}
     V_\gamma^\pi(x) + \mathbb{E}_\pi\left\lbrack\sum_{t=1}^\infty (\gamma^\prime-\gamma) (\gamma^\prime)^{t-1} V_\gamma^\pi(x_t)  \; \middle| \; x_0=x  \right\rbrack.
    \label{eq:undiscounted-rewards}
\end{align}
When $\gamma'=1$, $\rho_{x,\gamma,\gamma'}^\pi$ might be undefined. However, Eqn~\eqref{eq:undiscounted-rewards} still holds if assumptions \textbf{A.1} and \textbf{A.2} are satisfied. 
\end{restatable}

\subsection{Decomposing the full gradient $\nabla_\theta V_{\gamma'}^{\pi_\theta}(x)$}

Lemma~\ref{lemma:value-mixture} highlights that $V_{\gamma^\prime}^\pi(x)$ depends on $\pi$ in two aspects: \textbf{(1)} the value functions $V_\gamma^\pi(x^\prime),x^\prime\in\mathcal{X}$; \textbf{(2)} the state-dependent distribution $\rho_{x,\gamma,\gamma^\prime}^\pi(x^\prime)$. Let $\pi_\theta$ be a parameterized policy. For conceptual clarity, we can write $V_{\gamma'}^{\pi_\theta}(x) = F( V_\gamma^{\pi_\theta},\rho_{x,\gamma,\gamma'}^{\pi_\theta})$ with a function $F: \mathbb{R}^{\mathcal{X}}\times\mathbb{R}^{\mathcal{X}}\rightarrow \mathbb{R} $. Though this function is essentially the inner product, i.e., $F(V,\rho)=V^T\rho$, notationally, it helps stress that $V_{\gamma'}^{\pi_\theta}(x)$ depends on $\theta$ through two vector arguments. Now, we can decompose $\nabla_\theta V_{\gamma'}^{\pi_\theta}(x)$.
\begin{restatable}{lemma}{lemmagrad}\label{lemma:grad}
The full gradient $\nabla_\theta V_{\gamma'}^{\pi_\theta}(x)$ can be decomposed into the sum of two partial gradients as follows,
\begin{align*}
    & \left(\partial_V F(V,\rho)\right)^T \nabla_\theta V_\gamma^{\pi_\theta} + \left(\partial_\rho F(V,\rho) \right)^T \nabla_\theta \rho_{x,\gamma,\gamma'}^{\pi_\theta} \\
    &=\underbrace{\mathbb{E}\left\lbrack \nabla_\theta V_\gamma^{\pi_\theta}(x^\prime)\right\rbrack}_{\text{first\ partial\ gradient}} + \underbrace{ \mathbb{E}\left\lbrack  V_\gamma^{\pi_\theta}(x^\prime)\nabla_\theta \log \rho_{x,\gamma,\gamma^\prime}^{\pi_\theta}(x^\prime) \right\rbrack}_{\text{second\ partial\ gradient}},
\end{align*}
where the above partial gradients are both evaluated at $V=V_\gamma^{\pi_\theta},\rho=\rho_{x,\gamma,\gamma'}^{\pi_\theta}$ and both expectations are with respect to $x'\sim \rho_{x,\gamma,\gamma'}^{\pi_\theta}$.
\end{restatable}
We argue that the second partial gradient introduces most challenges in practical optimization. Intuitively, this is because its unbiased estimator is equivalent to a REINFORCE gradient estimator which requires estimating discounted values that accumulate $V_\gamma^\pi(x^\prime)$ as `reward' under discount factor $\gamma^\prime$. By the premise of our work, this estimation would be difficult. We will detail the discussions in Appendix~\ref{appendix:kweight}.

The following result characterizes the first partial gradient.

\begin{restatable}{proposition}{proppartial}\label{prop:partial-grad} For any $\gamma<\gamma'<1$, the first partial gradient $(\partial_V F(V_\gamma^{\pi_\theta},\rho_{x,\gamma,\gamma'}^{\pi_\theta}))^T \nabla_\theta V_\gamma^{\pi_\theta}$  can be expressed as 
\begin{align}
    \mathbb{E}_{\pi_\theta}\left\lbrack \sum_{t=0}^\infty (\gamma')^t  Q_\gamma^{\pi_\theta}(x_t,a_t)\nabla_\theta \log\pi_\theta(a_t|x_t)  \; \middle| \;x_0=x\right\rbrack.
    \label{eq:first-partial-grad}
\end{align} 
When $\gamma'=1$, under assumptions \textbf{A.1} and \textbf{A.2}, the first partial gradient exists and is expressed as 
\begin{align}
    \mathbb{E}_{\pi_\theta}\left\lbrack \sum_{t=0}^{T} Q_\gamma^{\pi_\theta}(x_t,a_t)\nabla_\theta \log\pi_\theta(a_t|x_t)  \; \middle| \;x_0=x \right\rbrack
    \label{eq:undiscounted-first-partial-grad}.
\end{align} 
\end{restatable}

\paragraph{Connections to common deep RL heuristic.}
Many high-quality deep RL algorithms (see, e.g. \citealp{dhariwal2017openai,SpinningUp2018}) implement parameter updates 
 which are very similar to  Eqn~\eqref{eq:undiscounted-first-partial-grad}. As such, Proposition~\ref{prop:partial-grad} provides some insights on why implementing such a heuristic might be useful in practice: though in general Eqn~\eqref{eq:undiscounted-first-partial-grad} is not a gradient \citep{nota2019policy}, it is a partial gradient of $V_{\gamma'=1}^{\pi_\theta}(x)$, which is usually the objective of interest at evaluation time. Compared with the formula of vanilla PG  in Eqn~\eqref{eq:pg}, Eqn~\eqref{eq:undiscounted-first-partial-grad} offsets the \emph{over-discounting} by via a uniform average over states. 

However, it is worth noting that in deep RL practice, the definition of the evaluation horizon $T$ might slightly differ from that specified in \textbf{A.1}. In such cases, Proposition~\ref{prop:partial-grad} does not hold. By \textbf{A.1}, $T$ is the absorption time that defines when the MDP enters a terminal absorbing state. In many applications, however, for MDPs without a natural terminatal state, $T$ is usually enforced by an external time constraint which does not depend on states. In other words, an environment can terminate even when it does not enter any terminal state (see, e.g., \citealp{brockman2016openai} for such examples). To bypass this subtle technical gap, one idea is to incorporate time steps as part of the state $\tilde{x}\leftarrow [x,t]$. This technique was hinted at in early work such as \citep{schulman2015high} and empirically studied in \citep{pardo2018time}. In this case, the random absorbing time $T$ depends fully on the augmented states, and Proposition~\ref{prop:partial-grad} holds.

\subsection{Taylor expansions of partial gradients}

We now consider approximations to the first partial gradients \begin{align*}
    \left(\partial_V F(V_\gamma^{\pi_\theta},\rho_{x,\gamma,\gamma'}^{\pi_\theta})\right)^T \nabla_\theta V_\gamma^{\pi_\theta} = (\rho_{x,\gamma,\gamma'}^{\pi_\theta})^T \nabla_\theta V_\gamma^{\pi_\theta}.
\end{align*}
Since $\nabla_\theta V_\gamma^{\pi_\theta}$ does not depend on $\gamma'$, the approximation is effectively with respect to the weight vector $\rho_{x,\gamma,\gamma'}^{\pi_\theta}$.
Below, we show results for the $K$\textsuperscript{th} order approximation.

\begin{restatable}{proposition}{proprho}\label{pro:rho}
Assume $\gamma<\gamma'<1$. For any $x\in\mathcal{X}$, define the $K$\textsuperscript{th} Taylor expansion to $\rho_{x,\gamma,\gamma'}^\pi$ as
\begin{align*}
   \rho_{x,K,\gamma,\gamma'}^\pi =  \sum_{k=0}^K \left((\gamma^\prime-\gamma) \left(I-\gamma (P^\pi)^T\right)^{-1} (P^\pi)^T \right)^k  \delta_x.
\end{align*}
It can be shown that $V_{K,\gamma,\gamma'}^\pi(x)=(\rho_{x,K,\gamma,\gamma'}^\pi)^T V_\gamma^\pi$ and $\left\lVert \rho_{x,K,\gamma,\gamma'}^\pi - \rho_{K,\gamma,\gamma'}^\pi \right\rVert_\infty = O\left(\left(\frac{\gamma'-\gamma}{1-\gamma}\right)^{K+1}\right)$.
\end{restatable}

We build some intuitions about the approximations. Note that
in general we can write the partial gradient as a weighted mixture of \emph{local gradients} $Q_t\nabla_\theta \log \pi_\theta(a_t|x_t)$ where $Q_t\coloneqq Q_\gamma^{\pi_\theta}(x_t,a_t)$,
\begin{align}
    \mathbb{E}_\pi\left[\sum_{t=0}^\infty w_{K,\gamma,\gamma'}(t) Q_t \nabla_\theta \log \pi_\theta(a_t|x_t) \; \middle| \; x_0=x \right],\label{eq:state-weight}
\end{align}
for some weight function $w_{K,\gamma,\gamma'}(t)\in\mathbb{R}$. 
When $K\rightarrow \infty$, $\lim w_{K,\gamma,\gamma'}(t)=(\gamma')^t$ and we recover the original first partial gradient defined in Eqn~\eqref{eq:first-partial-grad}; when $K=0$, $w_{K,\gamma,\gamma'}(t)=\gamma^t$ recovers the vanilla PG in Eqn~\eqref{eq:pg}. For other values of $K$, we show the analytic weights $w_{K,\gamma,\gamma'}(t)$ in Appendix~\ref{appendix:kweight}. Similar to how $V_{K,\gamma,\gamma'}^\pi$ interpolates  $V_{\gamma}^\pi$ and $V_{\gamma'}^\pi$, here the $K$\textsuperscript{th} order expansion to the partial gradients interpolate the full partial gradients and vanilla PG. In practice, we might expect an intermediate value of $K$ achieve the best bias and variance trade-off of the update.

\section{Policy optimization with Taylor expansions}
\label{sec:algo}

Based on theoretical insights of previous sections, we propose two algorithmic changes to baseline algorithms. Based on Section~\ref{sec:taylor}, we propose Taylor expansion advantage estimation; based on Section~\ref{sec:opt}, we propose Taylor expansion update weighting. It is important to note that other algorithmic changes are possible, which we leave to future work.

\subsection{Baseline near on-policy algorithm}
We briefly introduce backgrounds for near on-policy policy optimization algorithms \citep{schulman2015trust,mnih2016asynchronous,schulman2017proximal,espeholt2018impala}. We assume that the data are collected under a behavior policy $(x_t,a_t,r_t)_{t=0}^\infty\sim \mu$, which is close to the target policy $\pi_\theta$. The on-policyness is ensured by constraining $D(\pi_\theta,\mu)\leq \epsilon$ for some divergence $D$ and threshold $\epsilon>0$. Usually, $\epsilon$ is chosen to be small such that little off-policy corrections are needed for estimating value functions. With data $(x_t,a_t,r_t)_{t=0}^\infty$, the algorithms estimate Q-functions $\hat{Q}_\gamma^{\pi_\theta}\approx Q_\gamma^{\pi_\theta}$. Then the estimates $\hat{Q}_\gamma^{\pi_\theta}(x,a)$ are used as plug-in alternatives to the Q-functions in the definition of gradient updates such as Eqn~\eqref{eq:pg} for sample-based updates.

\subsection{Taylor expansion Q-function estimation}

In Section~\ref{sec:taylor}, we discussed how to construct approximations to $Q_{\gamma'}^{\pi_\theta}$ using $Q_{\gamma}^{\pi_\theta}$ as building blocks. As the first first algorithmic change, we propose to construct the $K$\textsuperscript{th} order expansion $Q_{K,\gamma,\gamma^\prime}^{\pi_\theta}$ as a plug-in alternative to $Q_\gamma^{\pi_\theta}$ when combined with downstream optimization. Since $Q_{K,\gamma,\gamma^\prime}^{\pi_\theta} \approx Q_{\gamma'}^{\pi_\theta}$, we expect the optimization subroutine to account for an objective of a longer effective horizon.

In many baseline algorithms, we have access to a value function critic $V_\phi(x)$ and a subroutine which produces Q-function estimates $\hat{Q}_\gamma^{\pi_\theta}(x,a)$ (e.g., $\hat{Q}_\gamma^{\pi_\theta}(x_t,a_t)=\sum_{s= 0}^\infty\gamma^s r_{t+s}$). We then construct the $K$\textsuperscript{th} order expansion $\hat{Q}_{K,\gamma,\gamma'}^{\pi_\theta}(x,a)$ using $\hat{Q}_\gamma^{\pi_\theta}$. This procedure is similar to Algorithm~\ref{algo:estimate}
and we show the full algorithm in Appendix~\ref{appendix:gamma}. See also Appendix~\ref{appendix:exp} for further experimental details.
 
 \begin{algorithm}[h]
\label{algo:k-adv}
\begin{algorithmic}
\REQUIRE policy $\pi_\theta$ with parameter $\theta$ and $\alpha$ \\
\WHILE{not converged}
\STATE 1. Collect partial trajectories $(x_t,a_t,r_t)_{t=1}^T \sim \mu$.
\STATE 2. Estimate Q-functions $\hat{Q}_\gamma^{\pi_\theta}(x_t,a_t)$.
\STATE 3. Construct $K$\textsuperscript{th} order Taylor expansion estimator $\hat{Q}_{K,\gamma,\gamma^\prime}^{\pi_\theta}(x_t,a_t)$ using $\hat{Q}^{\pi_\theta}(x_t,a_t)$.
\STATE 4. Update the parameter via gradient ascent $\theta \leftarrow \theta + \alpha  \sum_{t=1}^T  \hat{Q}_{K,\gamma,\gamma}^{\pi_\theta}(x_t) \nabla_\theta \log\pi_\theta(a_t|x_t) $.
\ENDWHILE
\caption{Taylor expansion Q-function estimation}
\end{algorithmic}
\end{algorithm}

\subsection{Taylor expansion update weighting}

In Section~\ref{sec:opt}, we discussed Taylor expansions approximation $\rho_{x,K,\gamma,\gamma'}^{\pi_\theta}$ to the weight vector $\rho_{x,\gamma,\gamma'}^{\pi_\theta}$. As the second algorithmic change to the baseline algorithm, we update parameters in the direction of $K$\textsuperscript{th} order approximations to the partial gradient
$
    \theta\leftarrow \theta + \alpha \left(\rho_{x,K,\gamma,\gamma'}^{\pi_\theta}\right)^T \nabla_\theta V_\gamma^{\pi_\theta}
$. Eqn~\eqref{eq:state-weight} shows that the update effectively translates into adjusting the weight $w_t= w_{K,\gamma,\gamma'}(t)$. When combined with other components of the algorithm, the pseudocode is shown in Algorithm 3. Under this framework, the common deep RL heuristic could be recovered by setting $w_t=1$.

\begin{algorithm}[h]
\label{algo:weight-onpolicy}
\begin{algorithmic}
\REQUIRE policy $\pi_\theta$ with parameter $\theta$ and $\alpha$ \\
\WHILE{not converged}
\STATE 1. Collect partial trajectories $(x_t,a_t,r_t)_{t=1}^T\sim\mu$.
\STATE 2. Estimate Q-functions  $\hat{Q}_t=\hat{Q}_\gamma^{\pi_\theta}(x_t,a_t)$. 
\STATE 3. Compute weights for each state $w_t=w_{x_0,K,\gamma,\gamma'}(t)$, and average $g_\theta =  \sum_{t=1}^T w_t \hat{Q}_t \nabla_\theta \log\pi_\theta(a_t|x_t) $. 
\STATE 4. Update parameters $\theta\leftarrow\theta+\alpha g_\theta$.
\ENDWHILE
\caption{Taylor expansion update weighting}
\end{algorithmic}
\end{algorithm}

\section{Related work}

\paragraph{Discount factors in RL.} Discount factors impact RL agents in various aspects. A number of work suggest that RL problems with large discount factors are generally more difficult to solve \citep{jiang2016structural}, potentially due to increased  complexities of the optimal value functions or collapses of the action gaps \citep{lehnert2018value,laroche2018reinforcement}. However, optimal policies defined with small discounts can be very sub-optimal for RL objectives with a large discount factor. To entail numerical stability of using large discounts, prior work has suggested non-linear transformation of the Bellman targets for Q-learning \citep{pohlen2018observe,van2019general,kapturowski2018recurrent,van2019using}. However, when data is scarce, small discount factors might prove useful due to its implicit regularization effect \citep{amit2020discount}. 

As such, there is a trade-off mediated by choosing different values of discount factors. Similar trade-off effects are most well-known in the context of TD($\lambda$), where $\lambda\in[0,1]$ trades-off the bias and variance of the TD updates \citep{sutton1998introduction,kearns2000bias}. 

\paragraph{Adapting discount factors \& multiple discount factors.} In general, when selecting a single optimal discount factor for training is difficult, it might be desirable to adjust the discount during training. This could be achieved by  human-designed  \citep{prokhorov1997adaptive,franccois2015discount} or blackbox adaptation \citep{xu2018meta}. Alternatively, it might also be beneficial to learn with multiple discount factors at the same time, which could improve TD-learning  \citep{sutton1995td} or representation learning \citep{fedus2019hyperbolic}. Complementary to all such work, we study the connections between value functions defined with different discounts.

\paragraph{Taylor expansions for RL.} Recently in \citep{tang2020taylor}, Taylor expansions were applied to study the relationship between $V_\gamma^\pi$ and $V_\gamma^\mu$, i.e., value functions under the same discount factor but different policies $\pi\neq\mu$. This is useful in the context of off-policy learning. Our work is orthogonal and could be potentially combined with this approach.

\section{Experiments}
\label{sec:exp}

In this section, we evaluate the empirical performance of new algorithmic changes to the baseline algorithms. We focus on robotics control experiments with continuous state and action space. The tasks are available in OpenAI gym \citep{brockman2016openai}, with backends such as MuJoCo  \citep{todorov2012mujoco} and bullet physics \citep{coumans2015bullet}. We label the tasks as gym (G) and bullet (B) respectively. We always compare the undiscounted cumulative rewards evaluated under a default evaluation horizon $T=1000$. 

\paragraph{Hyper-parameters.} Throughout the experiments, we use the same hyper-parameters across all algorithms. The learning rate is tuned for the baseline PPO, and fixed across all algorithms. See Appendix~\ref{appendix:exp} for further details.

\subsection{Taylor expansion Q-function estimation}

\begin{figure}[t]
    \centering
    \subfigure[HalfCheetah(G)]{\includegraphics[keepaspectratio,width=.22\textwidth]{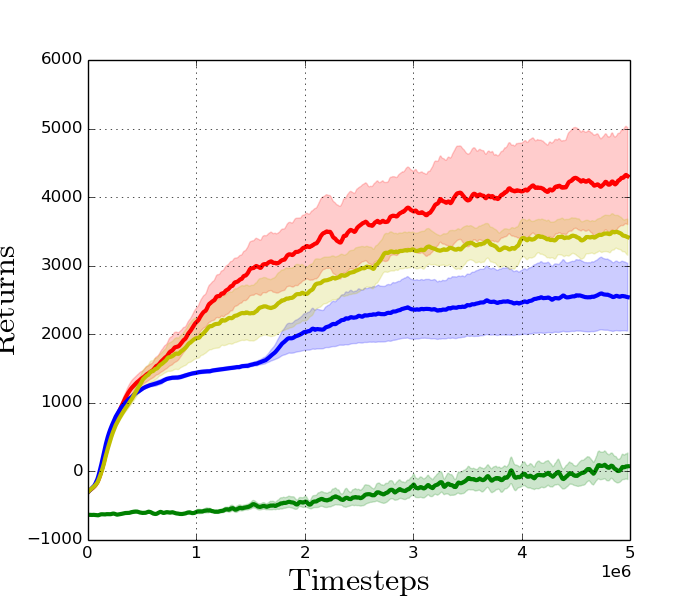}}
    \subfigure[Ant(G)]{\includegraphics[keepaspectratio,width=.22\textwidth]{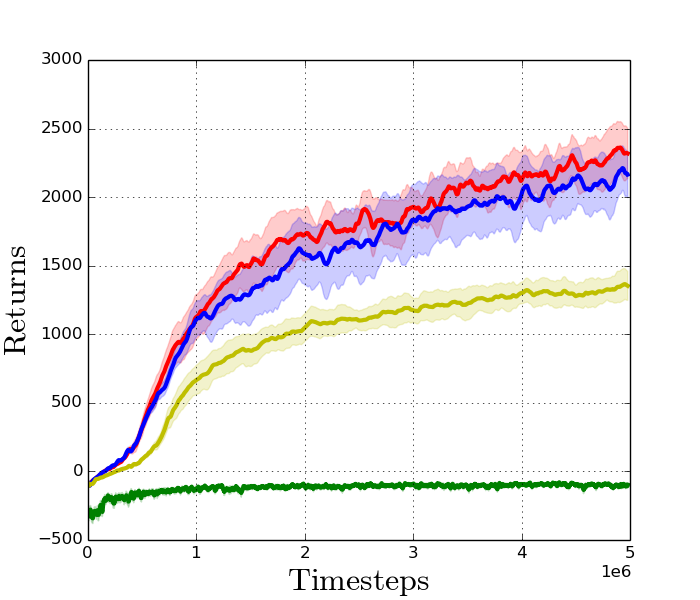}}
    \subfigure[Walker2d(G)]{\includegraphics[keepaspectratio,width=.22\textwidth]{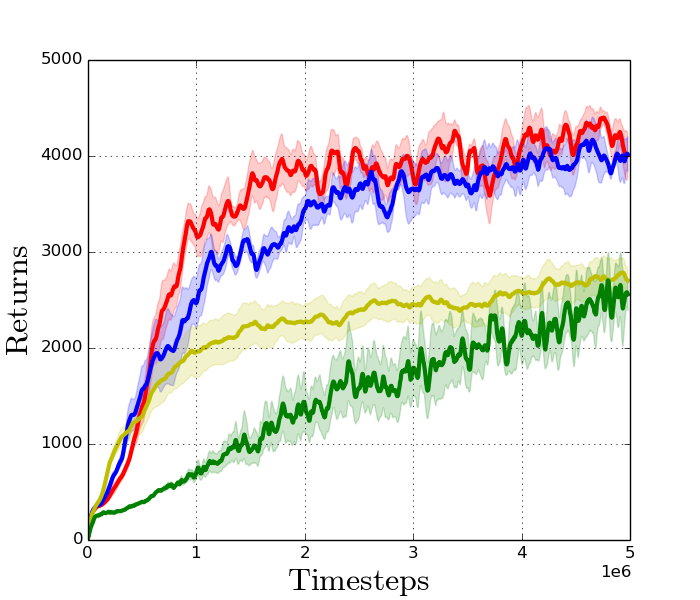}}
    \subfigure[HalfCheetah(B)]{\includegraphics[keepaspectratio,width=.22\textwidth]{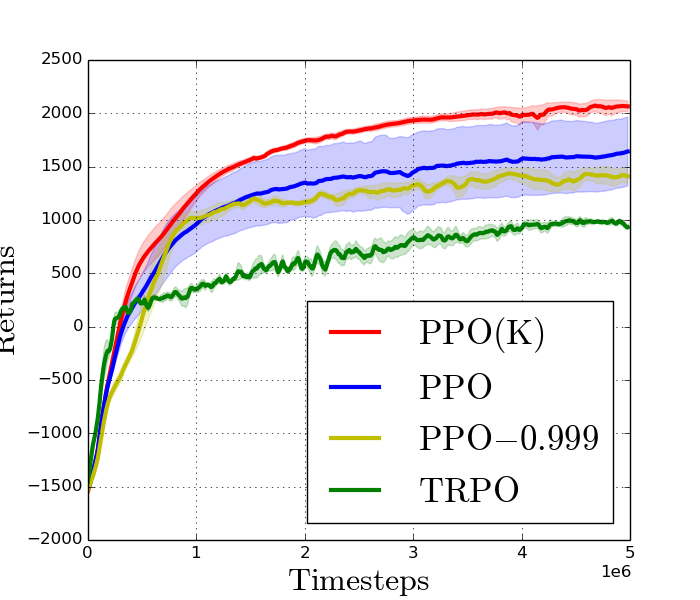}}
    \subfigure[Ant(B)]{\includegraphics[keepaspectratio,width=.22\textwidth]{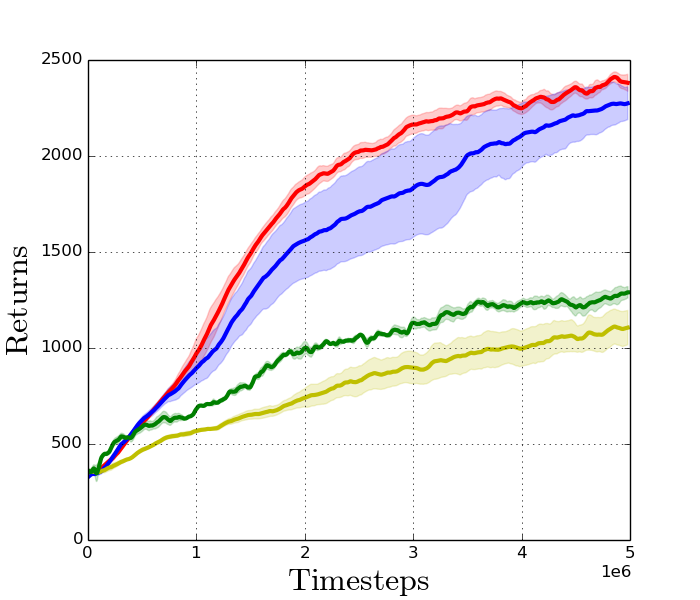}}
    \subfigure[Walker2d(B)]{\includegraphics[keepaspectratio,width=.22\textwidth]{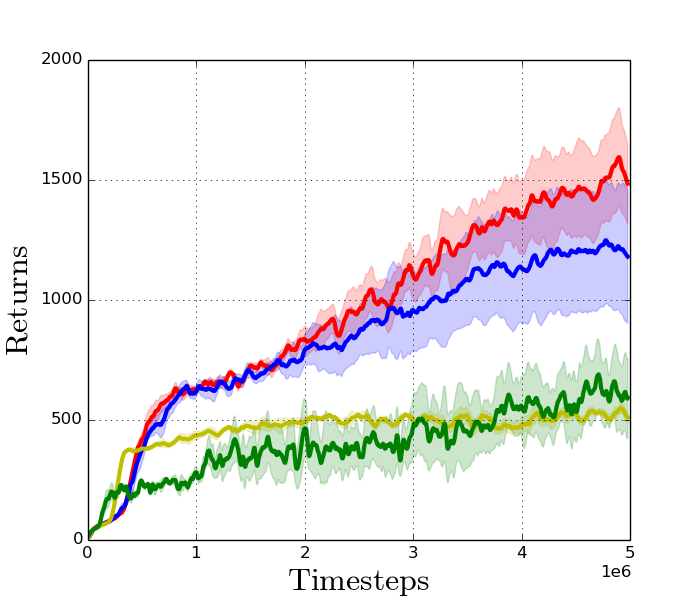}}
    \caption{Comparison of Taylor expansion Q-function estimation with other baselines. Each curve shows median $\pm$ std results across $5$ seeds. Taylor expansion outperforms PPO baselines with both lower and high discount factors.}
    \label{fig:onpolicy}
\end{figure}

We use $\hat{Q}_{K,\gamma,\gamma^\prime}^{\pi_\theta}(x,a)$ with $K=1$ as the Q-function estimator plug-in for the gradient update. When combining with PPO \citep{schulman2017proximal}, the resulting algorithm is named PPO($K$). We compare with the baseline PPO and TRPO \citep{schulman2015trust}. In practice, we consider a mixture of advantage estimator $\hat{Q}^{\pi_\theta}(x,a)=(1-\eta)\hat{Q}_\gamma^{\pi_\theta}(x,a)+\eta\hat{Q}_{K,\gamma,\gamma^\prime}^{\pi_\theta}(x,a)$ with $\eta\in[0,1]$ a constant that interpolates between the PPO (i.e., $\eta=0$) and PPO($K$). Note that though $\eta$ should be selected such that it balances the numerical scales of the two extremes, as a result, we usually find $\eta$ to work well when it is small in absolute scale ($\eta= 0.01$ works the best). 

\paragraph{Results.} In Figure~\ref{fig:onpolicy}, we compare a few baselines: \textbf{(1)} PPO with $\gamma=0.99$ (default); \textbf{(2)} PPO with high discount factor $\gamma=1-\frac{1}{T}=0.999$; \textbf{(3)} PPO with Taylor expansion based advantage estimator, PPO($K$). Throughout, we use a single hyper-parameter $\eta=0.01$. We see that in general, PPO($K$) leads to better performance (faster learning speed, better asymptotic performance or smaller variance across $5$ seeds). This shows Taylor expansion Q-function estimation could lead to performance gains across tasks, given that the hyper-parameter $\eta$ is carefully tuned. We provide a detailed ablation study on $\eta$ in Appendix~\ref{appendix:exp}, where we show how the overall performance across the benchmark tasks vary as $\eta$ changes from small to large values.

A second observation is that simply increasing the discount factor to $\gamma=1-\frac{1}{T}=0.999$ generally degrades the performance. This confirms  issue with instability of directly applying high discount factors which motivates this work.

We also compare with the open source implementation of \citep{romoff2019separating} in Appendix~\ref{appendix:exp}, where they estimate $\hat{Q}_{\gamma'}^\pi$ based on recursive bootstraps of Q-function differences. Conceptually, this is similar to Taylor expansions with $K=\infty$. We show that without a careful trade-off mediated by smaller $K$, this algorithm does not improve performance out of the box. 

\subsection{Taylor expansion update weighting}

As introduced in Section~\ref{sec:algo}, we weigh local gradients $\hat{Q}_t\nabla_\theta\log \pi_\theta(a_t|x_t)$ with $K$\textsuperscript{th} order expansion weights $w_{K,\gamma,\gamma'}(t)$. Here, we take $\gamma'=1-\frac{1}{T}$. Note that since $K=\infty$ corresponds to $\lim w_{K,\gamma,\gamma'}(t) = (\gamma')^t\approx 1$, this is very close to the commonly implemented PPO baseline. We hence expect the algorithm to work better with relatively large values of $K$ and set $K=100$ throughout experiments. In practice, we find the performance to be fairly robust in the choice of $K$. We provide further analysis and ablation study in Appendix~\ref{appendix:exp}. 

\begin{figure}[t]
    \centering
    \subfigure[HalfCheetah(G)]{\includegraphics[keepaspectratio,width=.22\textwidth]{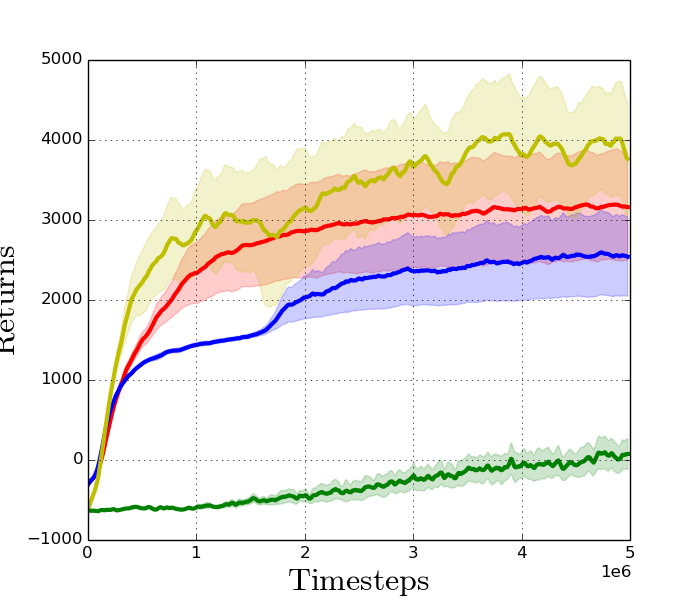}}
    \subfigure[Ant(G)]{\includegraphics[keepaspectratio,width=.22\textwidth]{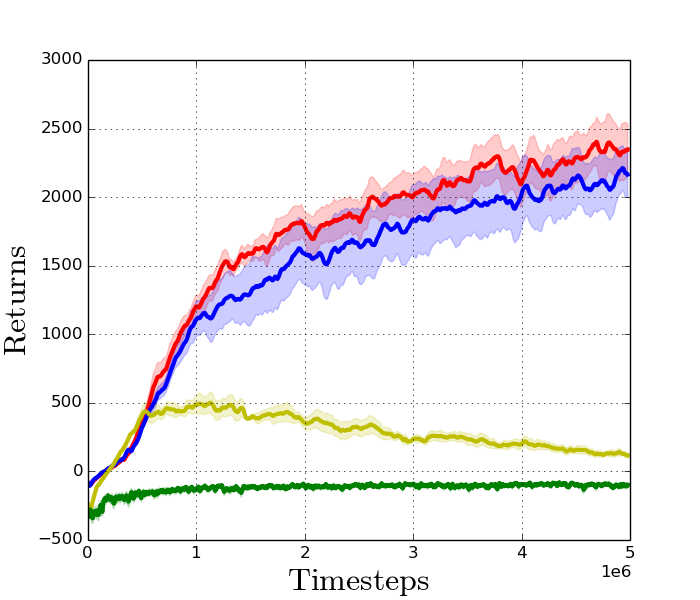}}
    \subfigure[Walker2d(G)]{\includegraphics[keepaspectratio,width=.22\textwidth]{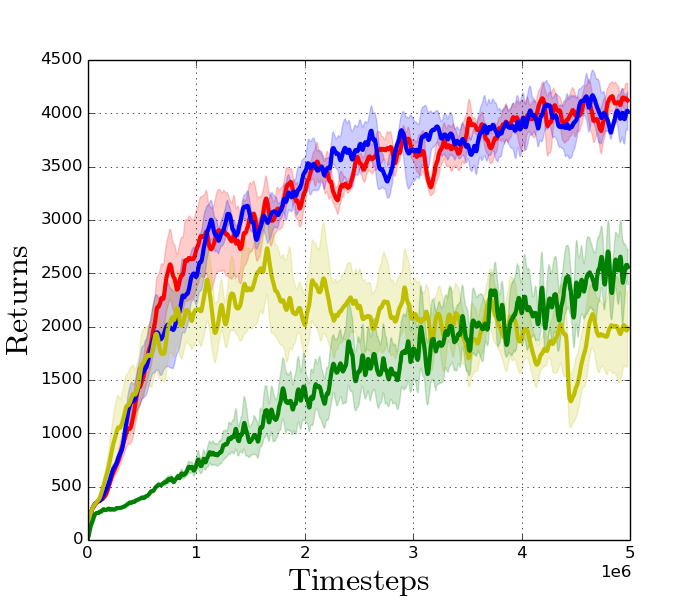}}
    \subfigure[HalfCheetah(B)]{\includegraphics[keepaspectratio,width=.22\textwidth]{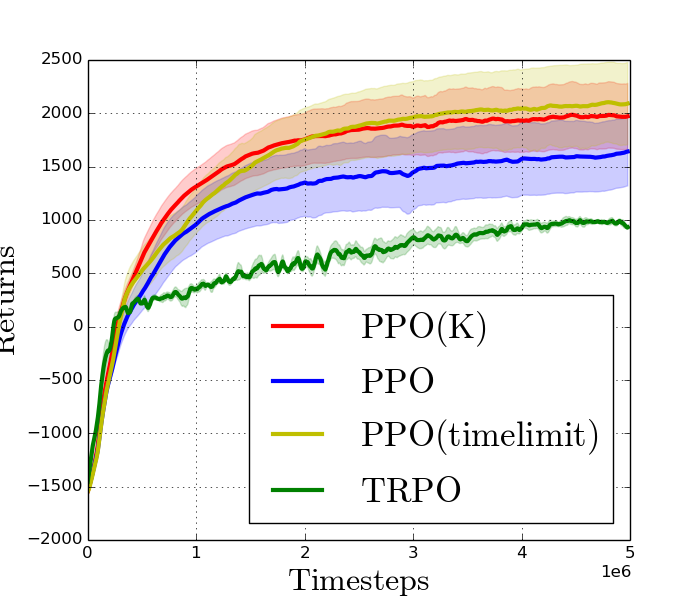}}
    \subfigure[Ant(B)]{\includegraphics[keepaspectratio,width=.22\textwidth]{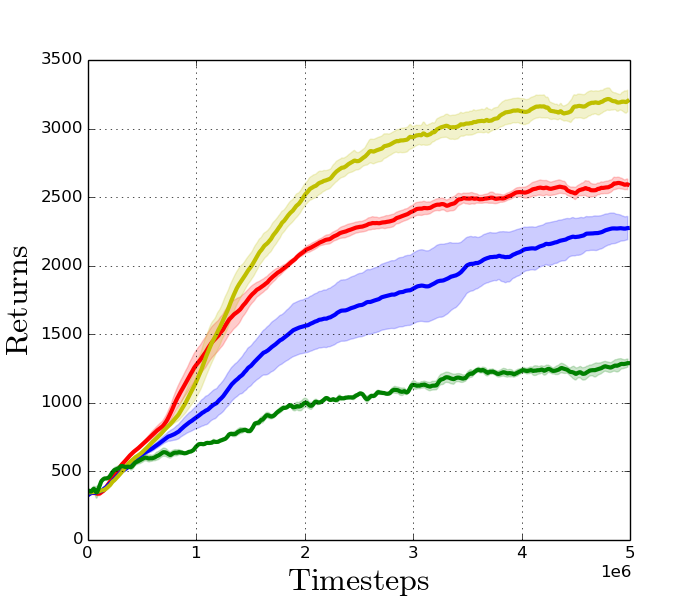}}
    \subfigure[Walker2d(B)]{\includegraphics[keepaspectratio,width=.22\textwidth]{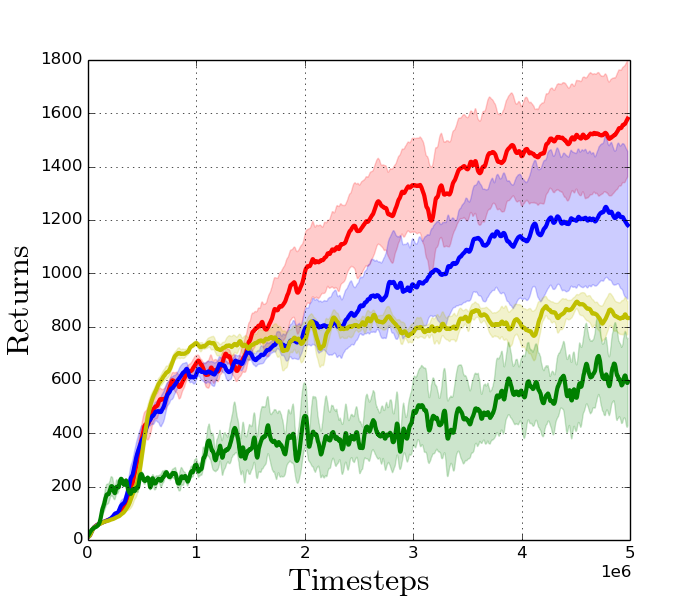}}
    \caption{Comparison of Taylor expansion update weighting with other baselines. Each curve shows median $\pm$ std results across $5$ seeds. Taylor expansion outperforms the default PPO baseline most stably.}
    \label{fig:onpolicy-state}
\end{figure}

\paragraph{Results.} We compare a few baselines: \textbf{(1)} default PPO; \textbf{(2)} PPO with time limit \citep{pardo2018time}. In this case, the states are augmented with time steps $\tilde{x}\leftarrow [x,t]$ such that the augmented states $\tilde{x}$ are Markovian; \textbf{(3)} PPO with Taylor expansion update weighting PPO($K$). In Figure~\ref{fig:onpolicy-state}, we see that in general, PPO($K$) and PPO with time limit outperform the baseline PPO. We speculate that the performance gains arise from the following empirical motivation: since the evaluation stops at $t=T$, local gradients close to $t=T$ should be weighed down because they do not contribute as much to the final objective. However, the default PPO ignores such an effect and weighs all updates uniformly. To tackle this issue, PPO($K$) explicitly weighs down the update while and PPO with time limit augments the state space to restore stationarity. Empirically, though in some cases PPO with time limit also outperforms PPO($K$), it behaves fairly unstably in other cases. 

\paragraph{Extensions to off-policy algorithms.} 
Above, we mainly focused on on-policy algorithms. The setup is simpler because the data are collected (near) on-policy. It is possible to extend similar results to off-policy algorithms \citep{mnih2015human,lillicrap2015continuous,fujimoto2018addressing,haarnoja2018soft}.  Due to the space limit, we present extended results in Appendix~\ref{appendix:exp}, where we show how to combine similar techniques to off-policy actor-critic algorithms such as TD3 \citep{fujimoto2018addressing} and SAC \citep{haarnoja2018soft} in  continuous control domains.

\section{Conclusion}
We have proposed a family of objectives that interpolate value functions defined with two discount factors. We have shown that similar techniques are applicable to other cumulative quantities defined through discounts, such as PG updates. This framework allowed us to achieve trade-off in estimating value functions or gradient updates, and led to empirical performance gains.

We also highlighted a new direction for bridging the gap between theory and practice: the gap between a fully discounted objective (in theory) and an undiscounted objective (in practice). By building a better understanding of this gap, we shed light on seemingly opaque heuristics which are necessary to achieve good empirical performance. We expect this framework to be useful for new practical algorithms.

\paragraph{Acknowledgements.} Yunhao thanks Tadashi Kozuno and Shipra Agrawal for discussions on the discrepancy between policy gradient theory and practices. Yunhao acknowledges the support from Google Cloud Platform for computational resources. The authors also thank Pooria Joulani for reviewing a draft of the paper.

\bibliographystyle{plainnat}
\bibliography{main}

\clearpage
\onecolumn

\begin{appendix}

\section*{APPENDICES: Taylor Expansions of Discount Factors}

\section{Proofs}
\label{appendix:proof}
\propstaylor*
\begin{proof}
Recall the Woodbury matrix identity
\begin{align*}
    (I-{\gamma^\prime}P^\pi)^{-1}=(I-\gamma P^\pi)^{-1}+(\gamma^\prime-\gamma)(I-\gamma P^\pi)^{-1}P^\pi (I-\gamma^\prime P^\pi)^{-1}.
\end{align*}
Recall the equality $V_{\gamma^\prime}^\pi=(I-\gamma^\prime P^\pi)^{-1} r^\pi$. By plugging in the Woodbury matrix identity, this immediate shows
\begin{align*}
    V_{\gamma^\prime}^\pi &= (I-\gamma P^\pi)^{-1} r^\pi +(\gamma^\prime-\gamma)(I-\gamma P^\pi)^{-1}P^\pi (I-\gamma^\prime P^\pi)^{-1} r^\pi \nonumber \\
    &= V_\gamma^\pi + (\gamma^\prime-\gamma)(I-\gamma P^\pi)^{-1}P^\pi V_{\gamma^\prime}^\pi.
\end{align*}
Now, observe that the second term involves $V_{\gamma^\prime}^\pi$. We can plug in the definition of $V_{\gamma^\prime}^\pi=(I-\gamma^\prime P^\pi)^{-1} r^\pi$ and invoke the Woodbury matrix identity again. This produces
\begin{align*}
    V_{\gamma^\prime}^\pi &=  V_\gamma^\pi +(\gamma^\prime-\gamma)(I-\gamma P^\pi)^{-1}P^\pi V_\gamma^\pi +   \left((\gamma^\prime-\gamma)(I-\gamma P^\pi)^{-1}P^\pi\right)^2 V_{\gamma^\prime}^\pi.
\end{align*}
By induction, it is straightforward to show that iterating the above procedure $K\geq  0$ times produces the following equalities
\begin{align*}
    V_{\gamma^\prime}^\pi = \sum_{k=0}^K \left((\gamma^\prime-\gamma) (I-\gamma P^\pi)^{-1} P^\pi \right)^k V_\gamma^\pi +
    &+ \underbrace{\left((\gamma^\prime-\gamma) (I-\gamma P^\pi)^{-1} P^\pi \right)^{K+1} V_{\gamma^\prime}^\pi}_{\text{residual}}.
\end{align*}
Consider the norm of the residual term. Since $P^\pi$ is a transition matrix, $\left\lVert P^\pi \right\rVert_\infty < 1$. As a result, $\left\lVert (I-\gamma P^\pi)^{-1}\right\rVert_\infty = \left\lVert \sum_{t=0}^\infty \gamma^t (P^\pi)^t \right\rVert_\infty < (1-\gamma)^{-1}$. This implies
\begin{align*}
    \left\lVert \left((\gamma^\prime-\gamma) (I-\gamma P^\pi)^{-1} P^\pi \right)^{K+1} V_{\gamma^\prime}^\pi \right\rVert_\infty < \left(\frac{\gamma^\prime-\gamma}{1-\gamma}\right)^{K+1} \cdot \frac{R_\text{max}}{1-\gamma^\prime}.
\end{align*}
When $\gamma<\gamma^\prime<1$, the residual norm decays exponentially and $\rightarrow 0$ as $K\rightarrow \infty$. This implies that the infinite series converges,
\begin{align*}
    V_{\gamma^\prime}^\pi = \sum_{k=0}^\infty \left((\gamma^\prime-\gamma) (I-\gamma P^\pi)^{-1} P^\pi \right)^k V_\gamma^\pi .
\end{align*}

\paragraph{Additional consideration when $\gamma'=1$.} When $\gamma'=1$, in order to ensure finiteness of $V_{\gamma'=1}^\pi$, we assume the following two conditions: \textbf{(1)} The Markov chain induced by $\pi$ is absorbing; \textbf{(2)} for any absorbing state $x$, $r^\pi(x)=0$.
Without loss of generality, assume there exists a single absorbing state. In general, the transition matrix $P^\pi$ can be decomposed as follows \citep{grinstead2012introduction,ross2014introduction},
\begin{align*} P^\pi = 
\begin{pmatrix}
  \tilde{P}^\pi & \tilde{p}^\pi \\ 
 0 & 1
\end{pmatrix},
\end{align*}
where $\tilde{P}^\pi\in\mathbb{R}^{(|\mathcal{X}|-1)|\mathcal{A}|\times (|\mathcal{X}|-1)|\mathcal{A}|}$ and $\tilde{p}^\pi \in\mathbb{R}^{\mathcal{X}-1)}$. Here, the first $\mathcal{X}-1$ states are transient and the last state is absorbing. For convenience, define $\tilde{r}^\pi$ as the reward vector $r^\pi$ constrained on the first $\mathcal{X}-1$ transient states.
We provide a few lemmas below.

\begin{restatable}{lemma}{lemma1}\label{lemma:1} The matrix $(I-\tilde{P})^\pi$ is invertible and its inverse is $(I-\tilde{P})^\pi = \sum_{k=0}^\infty (\tilde{P})^k$.
\end{restatable}
\begin{proof}
Define a matrix $N=\sum_{k=0}^\infty (\tilde{P}^\pi)^k$, then $N[x,y]$ defines the expected number of times it takes to transition from $x$ to $y$ before absorption. By definition of the absorbing chain, $N$ is finite. This further shows that $(I-\tilde{P}^\pi)$ is invertible, because
\begin{align*}
    N(I-\tilde{P}^\pi) = (I-\tilde{P}^\pi)N = I.
\end{align*}
\end{proof}

\begin{restatable}{lemma}{lemma2}\label{lemma:2} Let $f(A,B)$ be a matrix polynomial function of matrix $A$ and $B$. Then
\begin{align*}
    f\left(P^\pi,\left(I-\gamma P^\pi\right)^{-1}\right) = \begin{pmatrix}
      f\left(\tilde{P}^\pi,\left(I-\gamma \tilde{P}^\pi\right)^{-1}\right) & B \\ 
       0 & 1
    \end{pmatrix},
\end{align*}
where $B$ is some matrix. 
\end{restatable}
\begin{proof}
The intuition for the above result is that polynomial transformation preserves the \emph{block triangle} property of $P^\pi$ and $(I-\gamma P^\pi)^{-1}$. In general, we can assume 
\begin{align*}
    f(A,B) = \sum_{m,n\leq K} c_{m,n} A^m B^n,
\end{align*}
for some $K\geq 0$ and $c_{m,n}\in \mathbb{R}$ are scalar coefficients. First, note that $(P^\pi)^k,k\geq 0$ is of the form
\begin{align*}
    \left(P^\pi\right)^k = \begin{pmatrix}
       \left(\tilde{P}^\pi\right)^k & C \\ 
       0 & 1
    \end{pmatrix},
\end{align*}
for some matrix $C$. Since $(I-\gamma A)^{-1} = \sum_{k=0}^\infty A^k $ for $A\in\{P^\pi,\tilde{P}^\pi\}$, this implies that 
\begin{align*}
    (I-\gamma P^\pi)^{-1} = \sum_{k=0}^\infty (\gamma P^\pi)^k =  \begin{pmatrix}
       \left(\tilde{P}^\pi\right)^k & D \\ 
       0 & 1
    \end{pmatrix} = \begin{pmatrix}
       \left(I-\gamma \tilde{P}^\pi\right)^{-1} & D \\ 
       0 & 1
    \end{pmatrix},
\end{align*}
for some matrix $D$. The above two results show that both polynomials of $P^\pi$ and $(I-\gamma P^\pi)^{-1}$ are \emph{block upper triangle} matrices. It is then straightforward that
\begin{align*}
     \left(P^\pi\right)^m \left(  \left(I-\gamma P^\pi\right)^{-1}\right)^n   = \begin{pmatrix}
       \left(\tilde{P}^\pi\right)^m \left(  \left(I-\gamma \tilde{P}^\pi\right)^{-1}\right)^n  & E \\ 
       0 & 1
    \end{pmatrix},
\end{align*}
for some matrix $E$. Finally, since $f(P^\pi,(I-\gamma P^\pi)^{-1})$ is a linear combination of $ \left(P^\pi\right)^m \left(  \left(I-\gamma P^\pi\right)^{-1}\right)^n$, we conclude the proof.
\end{proof}

\begin{restatable}{lemma}{lemma3}\label{lemma:3} Under assumption $\textbf{(1)}$ and $\textbf{(2)}$, one could write the value function $V_{\gamma'=1}^\pi$ as
\begin{align*}
    V_{\gamma'=1}^\pi = \sum_{k=0}^\infty (P^\pi)^t r^\pi,
\end{align*}
where the infinite series on the RHS converges. In addition, for any transient state $x$, $V_{\gamma'=1}^\pi(x) = \left[\sum_{k=0}^\infty (\tilde{P}^\pi)^k \tilde{r}^\pi\right](x)$.
\end{restatable}
\begin{proof}
Recall that $V_{\gamma'}^\pi(x)\coloneqq \mathbb{E}\left\lbrack \sum_{t=0}^\infty r_t  \; \middle| \; x_0=x  \right\rbrack$. Under assumption \textbf{(2)}, for any absorbing state $x$, $V_{\gamma'}^\pi(x)=0=\left[\sum_{k=0}^\infty (P^\pi)^k r^\pi\right](x)$. We can instead constrain the Markov chain to the transient states. For any transient state $x$, recall the definition of $N$ from Lemma~\ref{lemma:1}, it follows that
\begin{align*}
    V_{\gamma'=1}^\pi(x) = \sum_{y} N(\text{expected\ number\ of\ times\ in\ }y|x_0=x) r^\pi(y) =  \left[Nr^\pi\right](x) = \left[ \left(I-\tilde{P}^\pi\right)^{-1} \tilde{r}^\pi\right] (x)
    = \left[\sum_{k=0}^\infty (\tilde{P}^\pi)^k \tilde{r}^\pi \right](x).
\end{align*}
By Lemma~\ref{lemma:2}, this is equivalent to $\left[\sum_{k=0}^\infty (P^\pi)^t r^\pi\right](x)$. We thus complete the proof.
\end{proof}

\begin{restatable}{lemma}{lemma4}\label{lemma:4} The following holds for any $\gamma<1$,
\begin{align}
    \left(I-\tilde{P}^\pi\right)^{-1} &=  \left(I-\gamma\tilde{P}^\pi-(1-\gamma)\tilde{P}^\pi\right)^{-1}\nonumber \\ &= \sum_{k=0}^K \left((1-\gamma) \left( I-\gamma \tilde{P}^\pi \right)^{-1} \tilde{P}^\pi \right)^k \left(I-\gamma \tilde{P}^\pi\right)^{-1} + \left((1-\gamma) \left( I-\gamma \tilde{P}^\pi \right)^{-1} \tilde{P}^\pi  \right)^{K+1}\left(I-\tilde{P}^\pi\right)^{-1} \nonumber \\
    &= \sum_{k=0}^\infty \left((1-\gamma) \left( I-\gamma \tilde{P}^\pi \right)^{-1} \tilde{P}^\pi \right)^k \left(I-\gamma \tilde{P}^\pi\right)^{-1}.
    \label{eq:k-expansion-undiscounted}
\end{align}
\end{restatable}
\begin{proof}
The first two lines derive from a straightforward application of Woodbury matrix identity to $(I-\tilde{P}^\pi)^{-1}$. This is valid because by Lemma~\ref{lemma:1}, $(I-\tilde{P}^\pi)$ is invertible. The convergence of the infinite series is guaranteed for all $\gamma<1$. To see why, recall that the finiteness of $N=\sum_{k=0}^\infty(\tilde{P}^{\pi})^k $ implies $(\tilde{P}^{\pi})^{K+1}\rightarrow 0$. We can bound the residual,
\begin{align*}
    \left\lVert \left((1-\gamma) \left( I-\gamma \tilde{P}^\pi \right)^{-1} \tilde{P}^\pi \right)^{K+1} \left(I-\tilde{P}^\pi\right)^{-1} \right\rVert_\infty \leq \left\lVert (\tilde{P}^{\pi})^{K+1} \left(I-\tilde{P}^\pi\right)^{-1} \right\rVert_\infty \rightarrow 0.
\end{align*}
\end{proof}

Finally, we combine results from the above to prove the main claim. First, consider the absorbing state $x$. Due to Assumption \textbf{(2)}, $V_\gamma^\pi(x)=0$ for any $\gamma\in[0,1]$. The matrix equalities in Proposition~\ref{prop:kth-expansion} holds in this case.

In the following, we consider any transient states $x$. By Lemma~\ref{lemma:3} and Lemma~\ref{lemma:4}
\begin{align*}
    \tilde{V}_{\gamma'=1}^\pi(x) &= \left[\sum_{k=0}^\infty (\tilde{P}^\pi)^k \tilde{r}^\pi \right](x) \\ &= \left[ \sum_{k=0}^K \left((1-\gamma) \left( I-\gamma \tilde{P}^\pi \right)^{-1} \tilde{P}^\pi \right)^k \left(I-\gamma \tilde{P}^\pi\right)^{-1} \tilde{r}^\pi + \left((1-\gamma) \left( I-\gamma \tilde{P}^\pi \right)^{-1} \tilde{P}^\pi  \right)^{K+1}\left(I-\tilde{P}^\pi\right)^{-1} \tilde{r}^\pi \right](x) 
\end{align*}
Now, notice that because the last entries of $r^\pi,V_\gamma^\pi,V_{\gamma'=1}^\pi$ are zero (for the absorbing state),
\begin{align*}
   \left[ \left(I-\gamma \tilde{P}^\pi\right)^{-1} \tilde{r}^\pi \right](x) = \left[ \left(I-\gamma P^\pi\right)^{-1} r^\pi \right](x).
\end{align*}
Combining with Lemma~\ref{lemma:2},
\begin{align*}
     \tilde{V}_{\gamma'=1}^\pi(x) &= \left[ \sum_{k=0}^K \left((1-\gamma) \left( I-\gamma P^\pi \right)^{-1} P^\pi \right)^k \left(I-\gamma P^\pi\right)^{-1} \tilde{r}^\pi + \left((1-\gamma) \left( I-\gamma P^\pi \right)^{-1} P^\pi  \right)^{K+1} V_{\gamma'=1}^\pi \right](x) \\
     &= \left[ \underbrace{\sum_{k=0}^K \left((1-\gamma) \left( I-\gamma P^\pi \right)^{-1} P^\pi \right)^k V_\gamma^\pi }_{\text{K-th\ order\ expansion}} + \underbrace{ \left((1-\gamma) \left( I-\gamma P^\pi \right)^{-1} P^\pi  \right)^{K+1} V_{\gamma'=1}^\pi}_{\text{residual}} \right](x).
\end{align*}
The residual term $\rightarrow 0$ as $K\rightarrow 0$ with similar arguments used for Lemma~\ref{lemma:4}. We hence conclude the proof.
\end{proof}

\properrork*
\begin{proof}
The proof follows directly from the residual term in Proposition~\ref{prop:taylor-expansion}. Recall that the residual term takes the form
\begin{align*}
    V_{\gamma^\prime}^\pi - V_{K,\gamma,\gamma^\prime}^\pi = \left((\gamma^\prime-\gamma) (I-\gamma P^\pi)^{-1} P^\pi \right)^{K+1} V_{\gamma^\prime}^\pi.
\end{align*}
Its infinity norm can be bounded as 
$(\frac{\gamma^\prime-\gamma}{1-\gamma})^{K+1} \frac{R_\text{max}}{1-\gamma^\prime}$
\end{proof}

\lemmamixture*

\begin{proof}
We will derive the above result with the matrix form. Recall by applying Woodbury inversion identity to $(I-\gamma^\prime P^\pi)^{-1} = (I-(\gamma^\prime-\gamma) P^\pi - \gamma P^\pi)^{-1}$, we get
\begin{align*}
    (I-\gamma^\prime P^\pi)^{-1} &= \sum_{k=0}^\infty \left((\gamma^\prime-\gamma) (I-\gamma P^\pi)^{-1} P^\pi \right)^k (I-\gamma P^\pi)^{-1} \\
    &= (I-\gamma P^\pi)^{-1} +  \sum_{k=1}^\infty \left((\gamma^\prime-\gamma) (I-\gamma P^\pi)^{-1} P^\pi \right)^k (I-\gamma P^\pi)^{-1} \\ 
    &= (I-\gamma P^\pi)^{-1} + (\gamma^\prime-\gamma) \sum_{k=1}^\infty \left((\gamma^\prime-\gamma) (I-\gamma P^\pi)^{-1} P^\pi \right)^k \cdot (I-\gamma P^\pi)^{-1} \cdot P^\pi (I-\gamma P^\pi)^{-1} \\
    &= (I-\gamma P^\pi)^{-1} + (\gamma^\prime-\gamma) (I-\gamma^\prime P^\pi)^{-1} \cdot P^\pi \cdot (I-\gamma P^\pi)^{-1}.
\end{align*}
Then, right multiply the above equation by $r^\pi$,
\begin{align*}
    V_{\gamma^\prime}^\pi &= V_\gamma^\pi + (\gamma^\prime-\gamma) (I-\gamma^\prime P^\pi)^{-1} P^\pi V_\gamma^\pi \\
    &= V_\gamma^\pi + (\gamma'-\gamma) \sum_{t=1}^\infty (\gamma')^{t-1} (P^\pi)^t V_\gamma^\pi.
\end{align*}
By indexing both sides at state $x$, we recover the following equality,
\begin{align*}
    V_{\gamma^\prime}^\pi(x) = V_\gamma^\pi(x) + \mathbb{E}_\pi\left\lbrack\sum_{t=1}^\infty (\gamma^\prime-\gamma) (\gamma^\prime)^{t-1} V_\gamma^\pi(x_t)  \; \middle| \; x_0=x  \right\rbrack.
\end{align*}
To derive the expression for $\rho_{x,\gamma,\gamma}^\pi$, note that also
\begin{align*}
    V_{\gamma'}^\pi = (I-\gamma' P^\pi)^{-1} r^\pi = (I-\gamma P^\pi)(I-\gamma' P^\pi)^{-1} (I-\gamma P^\pi)^{-1} r^\pi = \underbrace{(I-\gamma P^\pi)(I-\gamma' P^\pi)^{-1}}_{\text{weight\ matrix} W} V_\gamma^\pi,
\end{align*}
where we use the fact that $(I-\gamma P^\pi)$ commutes with $(I-\gamma' P^\pi)^{-1}$. Since we define $\rho_{x,\gamma,\gamma'}^\pi$ as such that $V_{\gamma'}^\pi(x) = \left(\rho_{x,\gamma,\gamma'}^\pi\right)^T V_\gamma^\pi$, we can derive the matrix form of $\rho_{x,\gamma,\gamma'}^\pi$ by indexing the $x$-th row of weight matrix $W$. This directly leads to the desired result
\begin{align*}
    \rho_{x,\gamma,\gamma'}^\pi = \left(I-\gamma (P^\pi)^T\right) \left(I-\gamma' (P^\pi)^T\right)^{-1}\delta_x.
\end{align*}

\end{proof}

\proppartial*

\begin{proof}
First we assume $\gamma'<1$, we will consider the extension to $\gamma'=1$ at the end of the proof. Recall that the policy gradient takes the following form,
\begin{align*}
   \nabla_\theta V_\gamma^{\pi_\theta}(x) = \mathbb{E}_{\pi_\theta}\left\lbrack \sum_{t=0}^\infty \gamma^t Q_\gamma^{\pi_\theta}(x_t,a_t) \nabla_\theta \log \pi_\theta(a_t|x_t)\; \middle| \; x_0=x. \right\rbrack
\end{align*}
We plug in the above, the partial derivative $\left(\partial_V F(V_\gamma^{\pi_\theta},\rho_{x,\gamma,\gamma'}^{\pi_\theta})\right)^T \nabla_\theta V_\gamma^{\pi_\theta}$ evaluates to the following
\begin{align*}
    & \nabla_\theta V_\gamma^{\pi_\theta}(x) + \mathbb{E}_{\pi_\theta}\left\lbrack (\gamma'-\gamma) \sum_{t=1}^\infty (\gamma')^{t-1} \nabla_\theta V_\gamma^{\pi_\theta}(x_t) \right\rbrack\\
    &= \mathbb{E}_{\pi_\theta}\left\lbrack \sum_{t=0}^\infty \gamma^t Q_\gamma^{\pi_\theta}(x_t,a_t) \nabla_\theta \log \pi_\theta(a_t|x_t)\; \middle| \; x_0=x \right\rbrack 
    \\ &+ \mathbb{E}_{\pi_\theta}\left\lbrack (\gamma'-\gamma)(\gamma')^{t-1} \sum_{t=1}^\infty \sum_{s=0}^\infty \gamma^s Q_\gamma^{\pi_\theta}(x_{t+s},a_{t+s}) \nabla_\theta \log \pi_\theta(a_{t+s}|x_{t+s}) \; \middle| \; x_0=x \right\rbrack \\
    &= \mathbb{E}_{\pi_\theta}\left\lbrack \sum_{t=0}^\infty \left(\underbrace{\gamma^t + \sum_{u=1}^t (\gamma'-\gamma)(\gamma')^{u-1}\gamma^{t-u}}_{\text{coefficient\ at\ time\ t}} \right) Q_\gamma^{\pi_\theta}(x_t,a_t) \nabla_\theta \log \pi_\theta(a_t|x_t)\; \middle| \; x_0=x. \right\rbrack 
\end{align*}
In the above, the coefficient term at time $t$ can be calculated by carefully grouping terms across different time steps. It can be shown that the coefficient term evaluates to $(\gamma')^t$ for all $t\geq 0$. This concludes the proof.

\paragraph{Alternative proof based on matrix notations.} We introduce an alternative proof based on matrix notations as it will make the extension to $\gamma'=1$ simpler.
First, note that 
\begin{align*}
    V_{\gamma'}^\pi = \left(I-\gamma' P^\pi\right)^{-1}r^\pi = \left(I-\gamma' P^\pi\right)^{-1} \left(I-\gamma P^\pi\right) \left(I-\gamma P^\pi\right)^{-1} r^\pi = \left(I-\gamma P^\pi\right) \left(I-\gamma' P^\pi\right)^{-1}  \left(I-\gamma P^\pi\right)^{-1} r^\pi,
\end{align*}
where for the second equality we exploit the fact that $\left(I-\gamma P^\pi\right)$ commutes with $\left(I-\gamma' P^\pi\right)^{-1}$. Now, notice that the above rewrites as 
\begin{align*}
    V_{\gamma'}^\pi = \underbrace{ \left(I-\gamma P^\pi\right) \left(I-\gamma' P^\pi\right)^{-1} }_{W_{\gamma,\gamma'}} V_\gamma^\pi
\end{align*}
where $W_{\gamma,\gamma'}$ is the weight matrix. This matrix is equivalent to the weighting distribution $\rho_{x,\gamma,\gamma'}^\pi$ by $W_{\gamma,\gamma'}[x] = \rho_{x,\gamma,\gamma'}^\pi$ where $A[x]$ is the $x$-th row of matrix $A$. The first partial gradient corresponds to differentiating $V_{\gamma'}^{\pi_\theta}$ only through $V_{\gamma}^{\pi_\theta}$. To make the derivation clear in matrix notations, let $\theta_i$ be the $i$-th component of the parameter $\theta$. Define $\nabla_{\theta_i} V_\gamma^{\pi_\theta} \in \mathbb{R}^{\mathcal{X}}$ such that $\nabla_{\theta_i} V_\gamma^{\pi_\theta}(x) = \nabla_{\theta_i} V_\gamma^{\pi_\theta}(x)$, This means the $i$-th component of the first partial gradient across all states is
\begin{align*}
    W_{\gamma,\gamma'} \nabla_{\theta_i} V_\gamma^{\pi_\theta} \in \mathbb{R}^{\mathcal{X}}.
\end{align*}
Let $G_\gamma^{\theta_i}\in\mathbb{R}^{\mathcal{X}}$ to be the vector of local gradient (for parameter $\theta_i$) such that $G_\gamma^{\theta_i}(x) =  \sum_{a} \nabla_{\theta_i}  \pi_\theta(a|x) Q_\gamma^{\pi_\theta}(x,a)$. Vanilla PG \citep{sutton2000policy} can be expressed as 
\begin{align*}
    \nabla_{\theta_i} V_\gamma^{\pi_\theta} = \left(I-\gamma P^\pi\right)^{-1} G_\gamma^{\theta_i}.
\end{align*}
We can finally derive the following,
\begin{align*}
W_{\gamma,\gamma'} \nabla_{\theta_i} V_\gamma^{\pi_\theta} &=
(I-\gamma P^\pi)(I-\gamma'P^\pi)^{-1}G_\gamma^{\theta_i} \\ &= (I-\gamma P^\pi)(I-\gamma'P^\pi)^{-1} (I-\gamma P^\pi)^{-1} G_\gamma^{\theta_i} \\
&= (I-\gamma P^\pi) (I-\gamma P^\pi)^{-1} (I-\gamma'P^\pi)^{-1} G_\gamma^{\theta_i} \\
&= (I-\gamma'P^\pi)^{-1} G_\gamma^{\theta_i}\\
\end{align*}
Now, consider the $x$-th component of the above vector. We have $\nabla_\theta [J(\pi_\theta,\pi_t)]_{\pi_t=\pi_\theta}$ is equal to
\begin{align*}
     \mathbb{E}_{\pi_\theta}\left[ \sum_{t=0}^\infty (\gamma')^t \sum_a \nabla_{\theta_i}  \pi_\theta(a|x_t) Q_\gamma^{\pi_\theta}(x_t, a) \; \middle| \; x_0=x \right] = \mathbb{E}_{\pi_\theta}\left[ \sum_{t= 0}^\infty (\gamma')^t  \nabla_{\theta_i}  \log  \pi_\theta(a_t|x_t) Q_\gamma^{\pi_\theta}(x_t, a) \; \middle| \; x_0=x \right]
\end{align*}
When concatenating the gradient for all component $\theta_ii$ of $\theta$, we conclude the proof.

\paragraph{Extensions to the case $\gamma'=1$.} Similar to the arguments made in the proof of Proposition~\ref{prop:kth-expansion}, under assumptions \textbf{A.1} and \textbf{A.2}, we can decompose the transition matrix $P^\pi$ as
\begin{align*}
P^\pi = \begin{pmatrix}
  \tilde{P} & \tilde{p} \\ 0 & 1
\end{pmatrix},    
\end{align*}
where the last state is assumed to be absorbing. Though $(I-\gamma' P^\pi)^{-1}$ for $\gamma'=1$ is in general not necessarily invertible, the matrix $(I-\tilde{P})^{-1}$ is invertible. Since $r^\pi(x)$ for the absorbing state $x$, we have deduced that $Q_\gamma^\pi(x.a)=V_\gamma^\pi(x)=0$, and accordingly $G_\gamma^{\theta_i}(x)=0$. As such, though $(I-\gamma' P^\pi)^{-1}$ for $\gamma'=1$ might be undefined, the multiplication $(I-\gamma' P^\pi)^{-1}G_\gamma^{\theta_i}$ is defined, with the last entry being $0$. Since at time $t=T$, the chain enters the absorbing states, all local gradient terms that come after $T$ are zero. As a result, the $x$-th component of $(I-\gamma' P^\pi)^{-1}G_\gamma^{\theta_i}$ is
\begin{align*}
    \mathbb{E}_{\pi_\theta}\left[ \sum_{t= 0}^{T} (\gamma')^t  \nabla_{\theta_i}  \log  \pi_\theta(a_t|x_t) Q_\gamma^{\pi_\theta}(x_t, a) \; \middle| \; x_0=x \right]
\end{align*}

\end{proof}

\proprho*

\begin{proof}
Recall from Lemma~\ref{lemma:value-mixture}, by construction,
\begin{align*}
    \rho_{x,\gamma,\gamma'}^\pi = \left(I-\gamma' (P^\pi)^T \right)^{-1}\left(I-\gamma (P^\pi)^T \right)\delta_x.
\end{align*}
Simialr to the case of primal space expansions in Section~\ref{sec:primal}, we construct the $K$\textsuperscript{th} order expansion to $\rho_{x,\gamma,\gamma'}^\pi$ via the expansion of the matrix $\left(I-\gamma (P^\pi)^T \right)^{-1}$. Recall that
\begin{align*}
    \left(I-\gamma' (P^\pi)^T\right)^{-1} = \sum_{k=0}^\infty \left((\gamma'-\gamma) (P^\pi)^T \left(I-\gamma (P^\pi)^T \right)^{-1} \right)^k \left(I-\gamma (P^\pi)^T\right)^{-1}.
\end{align*}
\end{proof}
When truncating the infinite series to the first $K+1$ terms, we derive the $K$\textsuperscript{th} order expansion $\rho_{x,K,\gamma,\gamma'}^\pi$,
\begin{align*}
     \left((\gamma'-\gamma) (P^\pi)^T \left(I-\gamma (P^\pi)^T \right)^{-1} \right)^k \left(I-\gamma (P^\pi)^T\right)^{-1}\left(I-\gamma (P^\pi)^T \right)\delta_x  = \sum_{k=0}^K \left((\gamma'-\gamma) (P^\pi)^T \left(I-\gamma (P^\pi)^T \right)^{-1} \right)^k \delta_x.
\end{align*}
Note that since 
\begin{align*}
    \left\lVert \sum_{k=K+1}^\infty \left((\gamma'-\gamma) (P^\pi)^T \left(I-\gamma (P^\pi)^T \right)^{-1} \right)^k \left(I-\gamma (P^\pi)^T\right)^{-1} \right\rVert_\infty \leq \left(\frac{\gamma'-\gamma}{1-\gamma}\right)^{K+1} \frac{1}{1-\gamma'}.
\end{align*}
This concludes the proof.

\section{Further results on Taylor expansions in the dual space}
\label{appendix:dual}

The dual representation of value function $V_{\gamma^\prime}^\pi(x)$ in Eqn~\eqref{eq:visitation-value} is $V_{\gamma^\prime}^\pi(x)=(1-\gamma^\prime)^{-1}(r^\pi)^T d_{x,\gamma^\prime}^\pi$ where $r^\pi,d_{x,\gamma^\prime}^\pi\in\mathbb{R}^{\mathcal{X}}$ are vector rewards and visitation distribution starting at state $x$. Here, we abuse the notation $d_{x,\gamma}^\pi$ to denote both a function and a vector, i.e., $d_{x,\gamma}^\pi(x^\prime)$ can be interpreted as both a function evaluation and a vector indexing. Given such a dual representation, one natural question is whether the $K$\textsuperscript{th} expansion in the primal space corresponds to some approximations of the discounted visitation distribution $d_{K,\gamma,\gamma^\prime}^\pi\approx d_{x,\gamma^\prime}^\pi$. Below, we answer in the affirmative.

 Let $\mathbf{\delta}_x \in \mathbb{R}^{\mathcal{X}}$ be the  one-hot distribution such that $[\mathbf{\delta}_x]_{x^\prime} = 1$ only when $x^\prime=x$. The visitation distribution satisfies the following balance equation in matrix form
\begin{align}
    d_{x,\gamma^\prime}^\pi = (1-\gamma^\prime) \mathbf{\delta}_x + \gamma^\prime  (P^\pi)^T d_{x,\gamma^\prime}^\pi.
\end{align}
Inverting the equation, we obtain an explicit expression for the visitation distribution $d_{x,\gamma^\prime}^\pi = (1-\gamma^\prime) (I-\gamma^\prime P^\pi)^{-1} \mathbf{\delta}_x$. Following  techniques used in the derivation of Propo~\ref{prop:taylor-expansion}, we can derive similar approximation results for dual variables. See Appendix~\ref{appendix:dual}.

\begin{restatable}{proposition}{propstaylordual}\label{prop:taylor-expansion-dual}
The following holds for all $K\geq0$,
\begin{align}
    d_{x,\gamma^\prime}^\pi &= \frac{1-\gamma^\prime}{1-\gamma} \sum_{k=0}^K \left((\gamma^\prime-\gamma) \left(I-\gamma (P^\pi)^T \right)^{-1} (P^\pi)^T\right)^k d_{x,\gamma}^\pi. \nonumber \\
    &+ \underbrace{\left((\gamma^\prime-\gamma) \left(I-\gamma (P^\pi)^T\right)^{-1} (P^\pi)^T \right)^{K+1} d_{x,\gamma^\prime}^\pi}_{\text{residual}}. \label{eq:fundamental-taylor-dual}
\end{align}
When $\gamma<\gamma^\prime<1$, the residual norm $\rightarrow 0$, which implies that the following holds
\begin{align}
    d_{x,\gamma^\prime}^\pi = \frac{1-\gamma^\prime}{1-\gamma} \sum_{k=0}^\infty ((\gamma^\prime-\gamma) \left(I-\gamma (P^\pi)^T\right)^{-1} (P^\pi)^T )^k d_{x,\gamma}^\pi. \label{eq:expansion}
\end{align}
\end{restatable}

\begin{proof}
Starting from the fixed point equation satisfied by $d_{\gamma^\prime}^\pi$, we can apply Woodbury inversion indentity
\begin{align*}
    d_{\gamma^\prime}^\pi &= (1-\gamma^\prime)\left(I-\gamma^\prime (P^\pi)^T\right)^{-1}\delta_x \\
    &= (1-\gamma^\prime) \sum_{k=0}^K \left((\gamma^\prime-\gamma)(I-\gamma P^\pi)^{-1} (P^\pi)^T \right)^k \delta_x + (1-\gamma^\prime) \left( (\gamma^\prime-\gamma)\left(I-\gamma (P^\pi)^T \right)^{-1} (P^\pi)^T \right)^K  \left(I-\gamma^\prime (P^\pi)^T\right)^{-1} \delta_x \\
    &= \frac{1-\gamma^\prime}{1-\gamma} \sum_{k=0}^K \left((\gamma^\prime-\gamma)\left(I-\gamma (P^\pi)\right)^{-1} P^\pi \right)^k d_\gamma^\pi + (1-\gamma^\prime) \left( (\gamma^\prime-\gamma)(I-\gamma P^\pi)^{-1} (P^\pi)^T \right)^K  d_{\gamma^\prime}^\pi
\end{align*}
The norm of the residual term could be bounded as 
\begin{align*}
    \left\Vert(1-\gamma^\prime) \left( (\gamma^\prime-\gamma)(I-\gamma P^\pi)^{-1} P^\pi \right)^K  d_{\gamma^\prime}^\pi \right\rVert_\infty \leq (1-\gamma^\prime) \left(\frac{\gamma^\prime-\gamma}{1-\gamma}\right)^{K+1} \rightarrow 0.
\end{align*}
\end{proof}
With a similar motivation as expansions in the primal space, we define the $K$\textsuperscript{th} order expansion by truncating to first $K+1$ terms,
\begin{align}
    d_{x,K,\gamma,\gamma^\prime}^\pi \coloneqq \frac{1-\gamma^\prime}{1-\gamma} \sum_{k=0}^K ((\gamma^\prime-\gamma) (I-\gamma P^\pi)^{-1} P^\pi)^k d_{x,\gamma}^\pi
\label{eq:k-expansion}
\end{align}

The following result formalizes the connection between the $K$\textsuperscript{th} order dual approximation to the visitation distribution $d_{K,\gamma,\gamma^\prime}^\pi$ and the primal approximation to the value function at state $x$, $V_{K,\gamma,\gamma^\prime}^\pi(x)$.

\begin{restatable}{proposition}{propdual}\label{prop:dual}
The $K$\textsuperscript{th} order primal and dual approximations are related by the following equality for any $K\geq 0$,
\begin{align}
    V_{K,\gamma,\gamma^\prime}^\pi(x) = (1-\gamma^\prime)^{-1} \left(d_{x,K,\gamma,\gamma^\prime}^\pi\right)^T r^\pi
    \label{eq:primal-dual}
\end{align}
\end{restatable}

\begin{proof}
The proof follows by expanding out the RHS of the equation. Recall the definition of $d_{K,\gamma,\gamma^\prime}^\pi$,
\begin{align*}
   (d_{K,\gamma,\gamma^\prime}^\pi)^T &= \frac{1-\gamma^\prime}{1-\gamma} \sum_{k=0}^K (d_\gamma^\pi)^T \left((\gamma^\prime-\gamma) \left(I-\gamma P^\pi\right)^{-1}\right)^k \\
   &=  (1-\gamma^\prime) \sum_{k=0}^K \delta_x^T (I-\gamma P^\pi)^{-1} \left((\gamma^\prime-\gamma) \left(I-\gamma P^\pi\right)^{-1}\right)^k \\
   &=(1-\gamma^\prime) \delta_x^T \left[\sum_{k=0}^K \left( \left (\gamma^\prime-\gamma) (I-\gamma P^\pi\right)^{-1} P^\pi \right)^k \right] \cdot \left(I-\gamma P^\pi\right)^{-1}.
\end{align*}
Now multiply the RHS by $r^\pi$ and recall that $V_\gamma^\pi=(I-\gamma P^\pi)^{-1}r^\pi$, we conclude the proof,
\begin{align*}
    \text{RHS} = \frac{1-\gamma^\prime}{1-\gamma} \delta_x^T \left[\sum_{k=0}^K \left( \left (\gamma^\prime-\gamma) (I-\gamma P^\pi\right)^{-1} P^\pi \right)^k \right] V_\gamma^\pi = (1-\gamma^\prime) \delta_x^T V_{K,\gamma,\gamma^\prime}^\pi = (1-\gamma^\prime) V_{K,\gamma,\gamma^\prime}^\pi(x).
\end{align*}

\end{proof}

Proposition~\ref{prop:dual} shows that indeed, the $K$\textsuperscript{th} order approximation of the value function is equivalent to the $K$\textsuperscript{th} order approximation of the visitation distribution in the dual space. It is instructive to consider the special case $K=1$.

\section{Details on Taylor expansion Q-function advantage estimation}
\label{appendix:gamma}

\begin{restatable}{proposition}{propadv}\label{prop:adv}
Let $Q_\gamma^\pi \in \mathbb{R}^{\mathcal{X}\times\mathcal{A}}$ be the vector advantage functions. Let $\bar{P}^\pi\in\mathbb{R}^{\left(\mathcal{X}\times\mathcal{A}\right)\times \left(\mathcal{X}\times\mathcal{A}\right)}$ be the transition matrix such that $\bar{P}^\pi(x,a,x',a')=\pi(x'|x')p(x'|x,a)$. Define the $K$\textsuperscript{th} order Taylor expansion of advantage as $
    Q_{K,\gamma,\gamma^\prime}^\pi \coloneqq  \sum_{k=0}^K ((\gamma^\prime-\gamma) (I-\gamma \bar{P}^\pi)^{-1} \bar{P}^\pi )^k Q_\gamma^\pi
    $. Then $\lim_{K\rightarrow\infty} Q_{K,\gamma,\gamma^\prime}^\pi=Q_{\gamma^\prime}^\pi$ for any $\gamma<\gamma^\prime<1$. 
\end{restatable}

\begin{algorithm}[h]
\label{algo:adv}
\begin{algorithmic}
\REQUIRE A trajectory $(x_t,a_t,r_t)_{t=0}^\infty\sim\pi$ and discount factors $\gamma<\gamma^\prime<1$\\
\STATE 1. Compute advantage function estimates $\hat{Q}_\gamma^\pi(x_t,a_t)$ for states on the trajectory. For example,  $\hat{Q}_\gamma^\pi(x_t,a_t)=\sum_{t^\prime\geq t} \gamma^{t^\prime-t} r_{t^\prime}$. One could also apply other alternatives (e.g., \citep{schulman2015high}) which potentiall reduce the variance of $\hat{Q}_\gamma^\pi(x_t,a_t)$.
\STATE 2. Sample $K$ random time $\tau_i,1\leq i\leq K$, all i.i.d. geometrically distributed $\tau_i\sim \text{Geometric}(1-\gamma)$.
\STATE 3. Return $\frac{(\gamma^\prime-\gamma)^K }{ (1-\gamma)^K } \hat{Q}_\gamma^\pi(x_\tau,a_\tau)$, where $\tau= \sum_{i=1}^K \tau_i$.
\caption{Estimating the $K$\textsuperscript{th} term of the expansion (Q-function)}
\end{algorithmic}
\end{algorithm}

\begin{proof}
The proof follows closely that of Taylor expansion based approximation to value functions in Proposition~\ref{prop:kth-expansion}. Importantly, notice that here we define $\bar{P}^\pi$, which differs from $P^\pi$ used in the derivation of value functions. In particular, $\bar{P}^\pi(x,a,y,b)=p(y|x,a)\pi(b|y)$ for any $x,y\in\mathcal{X},a,b\in\mathcal{A}$. Let $r$ be the vector reward function. The Bellmen equation for Q-function is
\begin{align*}
    Q_{\gamma'}^\pi = r + \gamma' \bar{P}^\pi Q_{\gamma'}^\pi.
\end{align*}
Inverting the equation and applying the Woodbury inversion identity,
\begin{align*}
    Q_{\gamma'}^\pi&=(I-\gamma' \bar{P}^\pi)^{-1} r = \sum_{k=0}^\infty \left((\gamma'-\gamma) \left(I-\gamma \bar{P}^\pi\right)^{-1} \bar{P}^\pi \right)^k Q_\gamma^\pi
\end{align*}
The above equality holds for all $\gamma<\gamma'<1$ due to similar convergence argument as in Proposition~\ref{prop:kth-expansion}. Truncating the infinite series at step $K$, we arrive at the $K$\textsuperscript{th} order expansion $Q_{K,\gamma,\gamma'}^\pi$. By construction, $\lim_{K\rightarrow\infty} Q_{K,\gamma,\gamma'}^\pi = Q_{\gamma'}^\pi$.
\end{proof}

\section{Details on  Taylor expansion update weighting}
\label{appendix:kweight}

\begin{restatable}{proposition}{propkvalue}\label{prop:kvalue}
The following is true for all $K\geq0$,
\begin{align*}
    \rho_{x,K,\gamma,\gamma'}(x') = \mathbb{I}[x'=x] + \mathbb{E}_{\pi}\left\lbrack \sum_{t=1}^\infty f(K,t,\gamma,\gamma^\prime) \mathbb{I}[x_t=x'] \; \middle| \; x_0=x \right\rbrack,
\end{align*}
Equivalently, the $K$\textsuperscript{th} order Taylor expansion of $V_{\gamma'}^\pi(x)$ is 
\begin{align}
    V_{K,\gamma,\gamma^\prime}^\pi(x) = V_\gamma(x) + \mathbb{E}_{\pi}\left\lbrack \sum_{t=1}^\infty f(K,t,\gamma,\gamma^\prime) V_\gamma^\pi(x_t) \; \middle| \; x_0=x \right\rbrack,
\end{align}
where $f(K,t,\gamma,\gamma^\prime)=\sum_{u=1}^{\min(K,t)}(\gamma^\prime-\gamma)^u \gamma^{t-u}\binom{t-1}{t-u}$ is a weight function. 
\end{restatable}
\begin{proof}
We start with a few lemmas.

\begin{restatable}{lemma}{lemmacom1}\label{lemma:com1} For any $n\geq 0, k\geq 1$, define a set of $k$-dimensional vector $\left\{x_1,...x_k | x_i\in \mathbb{Z}_{\geq 0}, \sum_{i=1}^k x_i = n\right\}$ and let $F(n,k)$ be the size of this set. Then \begin{align*}
    F(n,k)=\binom{n+k-1}{k-1}.
\end{align*}
\end{restatable}
\begin{proof}
By construction, the above set can be decomposed into smaller sets by fixing the value of $x_k$, i.e.,
\begin{align*}
    \left\{x_1,...x_k | x_i\in \mathbb{Z}_{\geq 0}, \sum_{i=1}^k x_i = n\right\} = \cup_{s=0}^n \left\{x_1,...x_{k-1},x_k | x_i\in \mathbb{Z}_{\geq 0}, \sum_{i=1}^k x_i = n,x_k=s\right\}
\end{align*}
Since these sets do not overlap, we have a recursive formula, $F(n,k)=\sum_{s=0}^n F(n-s,k-1)$. Starting from the base case $F(n,1)=1,\forall n\geq 0$, it is straightforward to prove by induction that for all $n\geq 0, k\geq 1$
\begin{align*}
    F(n,k) = \binom{n+k-1}{k-1}.
\end{align*}
\end{proof}

\begin{restatable}{lemma}{lemmacom2}\label{lemma:com2} Consider $V_{K+1,\gamma,\gamma'}^\pi-V_{K,\gamma,\gamma'}^\pi$ for $K\geq 0$. It can be shown that
\begin{align*}
   V_{K+1,\gamma,\gamma'}^\pi-V_{K,\gamma,\gamma'}^\pi = (\gamma'-\gamma)^{K+1} \left(\sum_{t=0}^\infty F(t,K+1) \left(P^\pi\right)^t \right) \left(P^\pi\right)^{K+1} V_\gamma^\pi.
\end{align*}
\end{restatable}
\begin{proof}
Starting with the definition,
\begin{align*}
     V_{K+1,\gamma,\gamma'}^\pi-V_{K,\gamma,\gamma'}^\pi &= \left(\left(\gamma'-\gamma\right)\left(I-\gamma P^\pi\right)^{-1}P^\pi\right)^{K+1} V_\gamma^\pi \\
     &= \left(\left(\gamma'-\gamma\right)\left(I-\gamma P^\pi\right)^{-1}\right)^{K+1} \left(P^\pi\right)^{K+1} V_\gamma^\pi,
\end{align*}
where for the second equality we use the fact that $P^\pi$ commutes with $(I-\gamma P^\pi)^{-1}$. Then consider $\left((I-\gamma P^\pi)^{-1}\right)^{K+1}$,
\begin{align*}
    \left(\left(I-\gamma P^\pi\right)^{-1}\right)^{K+1} &= \left(\sum_{t=0}^\infty \left(\gamma P^\pi\right)^t\right)^{K+1} =  \sum_{s_1\geq 0} ...\sum_{s_{K+1}\geq 0} \left(\gamma P^\pi\right)^{\sum_{i=1}^{K+1} s_i} =  \sum_{s=0}^\infty F(s,K+1)  \left(\gamma P^\pi\right)^s.
\end{align*}
Note that the last equality corresponds to a regrouping of terms in the infinite summation -- instead of summing over $s_1,...s_{K+1}$ sequentially, we count the number of examples such that $\sum_{i=1}^{K+1} s_i=s$ and then sum over $s$. This count is exactly $F(s,K+1)$ as defined in Lemma~\ref{lemma:com1}. Hence the proof is completed.
\end{proof}

With the above lemmas, we are ready to prove the final result. We start by summing up all the differences of expansions,
\begin{align*}
    V_{K,\gamma,\gamma'}^\pi &= V_{0,\gamma,\gamma'}^\pi + \sum_{k=0}^{K-1} \left( V_{k+1,\gamma,\gamma'}^\pi - V_{k,\gamma,\gamma'}^\pi \right) \\
    &= V_\gamma^\pi + \sum_{k=0}^{K-1} (\gamma'-\gamma)^{k+1} \left( \sum_{t=0}^\infty F(t,k+1) \left(\gamma P^\pi\right)^{t}\right) \left( P^\pi\right)^{k+1} V_\gamma^\pi \\
    &= V_\gamma^\pi + \sum_{t=0}^\infty \sum_{k=0}^{K-1} (\gamma'-\gamma)^{k+1} \gamma^{-k-1} F(t,k+1) \left(\gamma P^\pi\right)^{t+k+1} V_\gamma^\pi \\
    &=V_\gamma^\pi + \sum_{t=0}^\infty \sum_{u=1}^{K} (\gamma'-\gamma)^{u} \gamma^{-u} F(t,u)  \left(\gamma P^\pi\right)^{t+u} V_\gamma^\pi \\
    &=V_\gamma^\pi + \sum_{s=0}^\infty \sum_{u=1}^{K} (\gamma'-\gamma)^{u} \gamma^{-u} F(s-u,u) \left(\gamma P^\pi\right)^{s} V_\gamma^\pi \\
    &=V_\gamma^\pi + \sum_{s=1}^\infty \sum_{u=1}^{K} (\gamma'-\gamma)^{u} \gamma^{-u} F(s-u,u) \left(\gamma P^\pi\right)^{s} V_\gamma^\pi \\
    &=V_\gamma^\pi + \sum_{s=1}^\infty \sum_{u=1}^{K} (\gamma'-\gamma)^{u} \gamma^{s-u} \binom{s-1}{u-1} \left( P^\pi\right)^{s} V_\gamma^\pi \\
    &=V_\gamma^\pi + \sum_{s=1}^\infty \sum_{u=1}^{\min(K,s)} (\gamma'-\gamma)^{u} \gamma^{s-u} \binom{s-1}{u-1} \left( P^\pi\right)^{s} V_\gamma^\pi \\
\end{align*}
In the above derivation, we have applied the transformation $u=k+1,s=t+u$. Then we have modified the bound of the summation with the definition of $F(s-u,u)$ (in particular, if $s<u$, $F(s-u,u)=0$). If we index the $x$-th component of the vector, we recover the desired result.

\end{proof}

\subsection{Further discussions on the objectives}
Recall that the full gradient $\nabla_\theta V_{\gamma^\prime}^{\pi_\theta}(x)$ is
\begin{align*}
    \nabla_\theta V_{\gamma^\prime}^{\pi_\theta}(x) &= \mathbb{E}_{x^\prime\sim\rho_{\gamma,\gamma^\prime}^{\pi_\theta}(\cdot;x)}\left\lbrack \nabla_\theta V_\gamma^{\pi_\theta}(x^\prime)\right\rbrack  + \underbrace{ \mathbb{E}_{x^\prime\sim\rho_{\gamma,\gamma^\prime}^{\pi_\theta}(\cdot;x)}\left\lbrack  V_\gamma^{\pi_\theta}(x^\prime)\nabla_\theta \log \rho_{\gamma,\gamma^\prime}^{\pi_\theta}(x^\prime;x). \right\rbrack }_{\text{second\ term}}
\end{align*}
Consider the second term. Now, we derive this term in an alternative way which imparts more intuitions on why its estimation is challenging. Note that
\begin{align*}
    V_{\gamma^\prime}^{\pi_\theta}(x)=V_\gamma^{\pi_\theta}(x) + (\gamma^\prime-\gamma) \mathbb{E}_\pi\left\lbrack \sum_{t=1}^\infty (\gamma^\prime)^{t-1} V_\gamma^{\pi_\theta}(x_t) \right\rbrack
\end{align*}
The second term of the full gradient is equivalent to differentiating through the above expression, while keeping all $V_\gamma^{\pi_\theta}(x_t)$ fixed. This leads the following gradient
\begin{align*}
    \text{second\ term} = (\gamma^\prime-\gamma)(\gamma^\prime)^{-1}\mathbb{E}_\pi\left\lbrack \sum_{t=1}^\infty (\gamma^\prime)^t W_{\gamma,\gamma^\prime}^{\pi_\theta}(x_t)\nabla_\theta \log\pi_\theta(a_t|x_t) \right\rbrack.
\end{align*}
Here, we introduce $W_{\gamma,\gamma^\prime}^{\pi_\theta}(x_t)=\mathbb{E}_\pi\left\lbrack \sum_{s=0}^\infty (\gamma^\prime)^s V_\gamma^{\pi_\theta}(x_{t+s}) \right\rbrack $, which is equivalent to a value function that treats $V_{\gamma}^{\pi_\theta}(x)$ as rewards and with discount factor $\gamma^\prime$. Naturally, constructing an unbiased estimator of the second term of the full gradient requires estimating $W_{\gamma,\gamma^\prime}^{\pi_\theta}$, which is difficult in at least two aspects: \textbf{(1)} in practice, value functions are already estimated, which could introduce additional bias and variance; \textbf{(2)} as a premise of our work, estimating discounted values with discount factor $\gamma^\prime$ is challenging potentially due to high variance.

\section{Details on approximation errors with finite samples}
\label{appendix:bound}

Intuitively, as $K$ increases, the $K$\textsuperscript{th} order expansion $V_{K,\gamma,\gamma^\prime}^\pi$ approximates $V_{\gamma^\prime}^K$ more accurately in expectation. However, in practice where all constituent terms of the approximation are built from the same batch of data, the variance might negatively impact the accuracy of the estimate.

To formalize such intuitions, we characterize the bias and variance trade-off under the phased TD-learning framework \citep{kearns2000bias}. Consider estimating the value function $V_\gamma^\pi(x)$ under discount $\gamma$, with estimator $\hat{V}_\gamma^\pi(x)$. At each iteration $t$, let $\Delta_t^\gamma\coloneqq \max_{x\in\mathcal{X}}|V_\gamma^\pi(x)-\hat{V}_\gamma^\pi(x)|$ be the absolute error of value function estimates $\hat{V}_\gamma^\pi$. Assume from each state $x$, there are independent $n$ trajectories generated under $\pi$, \citep{kearns2000bias} shows that commonly used TD-learning methods (e.g. TD($\lambda$)) have error bounds of the following form with probability $1-\delta$,
\begin{align}
    \Delta_t^\gamma \leq A(\gamma,\delta) + B(\gamma) \Delta_{t-1}^\gamma\label{eq:value-error}.
\end{align}
Here, the factor $A(\gamma,\delta)$ is an error term which characterizes the errors arising from the finite sample size $n$. As $n\rightarrow\infty$, $A(\gamma,\delta)\rightarrow 0$; the constant $B(\gamma)$ is a contraction coefficient that shows how fast the error decays in expectation. See Appendix~\ref{appendix:bound} for details. 

With the calculations of estimators $\hat{V}_\gamma^\pi(x)$ as a subroutine, we construct the $n$-sample $K$\textsuperscript{th} order estimator $\hat{V}_{K,\gamma,\gamma^\prime}^\pi(x)$,
\begin{align}
    \hat{V}_{K,\gamma}(x_0) = \sum_{k=0}^K \frac{1}{n} \sum_{i=1}^n (\gamma^\prime-\gamma)^k \hat{V}_\gamma^\pi (x_{i,k}),
    \label{eq:kth-order-estimator}
\end{align}
where $x_{i,k}$ is sampled from $(P^\pi\cdot d_\gamma^\mu)^k(\cdot;x)$. Note that if $K=0$, 
Eqn~\eqref{eq:kth-order-estimator} reduces to $\frac{1}{n} \sum_{i=1}^n \hat{V}_\gamma^\pi(x_0)$, the estimator analyzed by \citep{kearns2000bias}. We are interested in the error 
$\Delta_{K,t}^\gamma\coloneqq \max_{x\in\mathcal{X}} |V_{\gamma^\prime}^\pi(x) - \hat{V}_{K,\gamma,\gamma^\prime}^\pi(x)|$, measured against the value function of discount $\gamma^\prime$. The following summarizes how errors propagate across iterations,
\begin{restatable}{proposition}{properror}\label{prop:error} Assume all samples $x_{i,k}$ are generated independently.
 Define a factor $\epsilon\coloneqq \frac{1-(\gamma^\prime-\gamma)^{K+1}}{1-(\gamma^\prime-\gamma)}$. Then with probability at least $1-2\delta$ if $K\geq 1$ and probability $1-\delta$ if $K=0$, the following holds\footnote{The error bounds could be further improved, e.g., by adapting the concentration bounds at different steps $1\leq k\leq K$. Note that its purpose is to illustrate the bias and variance trade-off induced by the Taylor expansion order $K$.},
\begin{align}
    \Delta_{K,t}^\gamma &\leq  \underbrace{\epsilon(A(\gamma,\delta)+U)}_{\text{finite\ sample\ error}} + \underbrace{ E(\gamma,\gamma^\prime,K)}_{\text{expected\ gap\ error}} + \underbrace{\epsilon B(\gamma)}_{\text{contraction\ coeff}}\Delta_t^\gamma,
    \label{eq:error-decomposition}
\end{align}
where $U = \sqrt{2\log \frac{2(K+1)}{\delta} / n}$ for $K\geq 1$ and $U=0$ if $K=0$. The expected gap error $E(\gamma,\gamma^\prime,K)=\left(\frac{\gamma'-\gamma}{1-\gamma}\right)^{K+1}\frac{R_\text{max}}{1-\gamma}$ is defined in Proposition~\ref{prop:kth-expansion}.
\end{restatable}

\begin{proof}

Recall the results from \citep{kearns2000bias}: Let $\Delta_t^\gamma \coloneqq \max_{x\in\mathcal{X}}|V_\gamma^\pi(x) - \hat{V}_\gamma^\pi(x)|$. Then with probability at least $1-\delta$, the following holds
\begin{align*}
    \Delta_t^\gamma \leq A(\gamma,\delta) + B(\gamma)\Delta_{t-1}^\gamma.
\end{align*}
In the following, we condition all analysis on the event set that the above inequality holds. Now, using $\hat{V}_\gamma^\pi(x)$ as a subroutine, define the estimator for the $K$\textsuperscript{th} Taylor expansion as in Eqn~\eqref{eq:kth-order-estimator},
\begin{align*}
    \hat{V}_{K,\gamma}(x_0) = \sum_{k=0}^K \frac{1}{n} \sum_{i=1}^n (\gamma^\prime-\gamma)^k \hat{V}_\gamma^\pi (x_{i,k}).
\end{align*}
Define the error $\Delta_{K,t}^\gamma \coloneqq \max_{x\in\mathcal{X}}|V_{\gamma^\prime}^\pi(x) - \hat{V}_{K,\gamma}^\pi(x)|$, which is measured against the value function $V_{\gamma^\prime}^\pi(x)$ with a higher discount factor $\gamma^\prime$. Consider for a given starting state $x_0$,
\begin{align*}
    |V_{\gamma^\prime}^\pi(x_0) - \hat{V}_{K,\gamma}^\pi(x_0)| 
    &= V_{\gamma^\prime}^\pi(x_0) - V_{K,\gamma}^\pi(x_0) + V_{K,\gamma}^\pi(x_0) - \hat{V}_{K,\gamma}^\pi(x) \nonumber \\
    &\leq |V_{\gamma^\prime}^\pi(x_0) - V_{K,\gamma}^\pi(x_0)| +  V_{K,\gamma}^\pi(x_0) - \hat{V}_{K,\gamma}^\pi(x_0) \nonumber \\
    &\leq E(\gamma,\gamma^\prime,K) +  \underbrace{V_{K,\gamma}^\pi(x_0) - \mathbb{E}\left\lbrack\hat{V}_{K,\gamma}^\pi(x_0)\right\rbrack}_{\text{second\ term}} + \underbrace{ \mathbb{E}\left\lbrack\hat{V}_{K,\gamma}^\pi(x_0)\right\rbrack - \hat{V}_{K,\gamma}^\pi(x_0)}_{\text{third\ term}}. 
\end{align*}
Now, we bound each term in the equation above. Recall $\epsilon\coloneqq  \sum_{k=0}^K (\gamma^\prime-\gamma)^k = \frac{1-(\gamma^\prime-\gamma)^{K+1}}{1-\gamma^\prime+\gamma}$. The second term is bounded as follows
\begin{align*}
    V_{K,\gamma}^\pi(x_0) -  \mathbb{E}\left\lbrack\hat{V}_{K,\gamma}^\pi(x_0)\right\rbrack \leq \epsilon\Delta_t^\gamma.
\end{align*}
The third term is bounded by applying concentration bounds. Recall that the estimator $\hat{V}_{K,\gamma}^\pi(x_0)\coloneqq  \sum_{k=0}^K \frac{1}{n} \sum_{i=1}^n (\gamma^\prime-\gamma)^k \hat{V}_\gamma^\pi (x_{i,k})$ decomposes into $K+1$ estimators, each being an average over $n$ i.i.d. samples drawn from the $K$\textsuperscript{th} step visitation distribution $(P^\pi\cdot d_\gamma^\pi)^k,0\leq k\leq K$. Applying similarly naive techniques in \citep{kearns2000bias}, we bound each of the $K+1$ terms individually and then take a union bound over all $K+1$ terms. This implies that, with probability at least $1-\delta$, the following holds
\begin{align*}
        \mathbb{E}\left\lbrack\hat{V}_{K,\gamma}^\pi(x_0)\right\rbrack - \hat{V}_{K,\gamma}^\pi(x_0) \leq \epsilon U = \epsilon \sqrt{2\log \frac{2(K+1)}{\delta} / n}.
\end{align*}
Aggregating all results, we have
\begin{align*}
    |V_{\gamma^\prime}^\pi(x_0) - \hat{V}_{K,\gamma}^\pi(x_0)| &\leq E(\gamma,\gamma^\prime,K) + \epsilon\Delta_t^\gamma +\epsilon U \\
    &\leq \epsilon(A(\gamma,\delta) + U) + E(\gamma,\gamma^\prime,K) + \epsilon B(\gamma)\Delta_{t-1}^\gamma.
\end{align*}
This holds with probability at least $(1-\delta)^2\geq 1-2\delta$.

\end{proof}

\paragraph{Bias-variance trade-off via $K$.}
The error terms come from two parts: the first term contains  errors $A(\gamma,\delta)$ in the subroutine estimator $\hat{V}_\gamma^\pi(x)$, and its propagated errors through the sampling of $K$\textsuperscript{th} order approximations for $1\leq k\leq K$ (shown via the multiplier $\epsilon$). This first term also contains $U$, a concentration bound that scales with $O(\sqrt{\log K})$, which shows that the variance of the overall estimator grows with $K$. This first error term scales with $\sqrt{n}$ and vanishes as the number of samples increases. The second term is due to the gap between the expected $K$\textsuperscript{th} order Taylor expansion and $V_{\gamma^\prime}^\pi(x_0)$, which decreases with $K$ and does not depend on sample size $n$. The new contraction coefficient is $\epsilon B(\gamma)$, where it can be shown that $\epsilon \in [1,\frac{1}{1-\gamma^\prime+\gamma}]$. Since typical estimators have $B(\gamma)\leq \gamma$, in general $\epsilon B(\gamma)<1$ and the error contracts with respect to $\Delta_t$. In general, the contraction becomes slower as $K$ increases. For example, for TD($\lambda$), $B(\gamma)=\frac{(1-\lambda)\gamma}{1-\gamma\lambda}$.

\section{Further experiment details}
\label{appendix:exp}

Below, we provide further details on experiment setups along with additional results.

\subsection{Further details on the toy example}

We presented a toy example that highlighted the trade-off between bias and variance, mediated by the order parameter $K$. Here, we provide further details of the experiments.
\paragraph{Toy MDP.} We consider tabular MDPs with $|\mathcal{X}|=10$ states and $|\mathcal{A}|=2$ actions. The transition table $p(y|x,a)$ is drawn from a Dirichlet distribution $(\alpha,\ldots,\alpha)$ for $\alpha=0.01$. Here, $\alpha$ is chosen such that the MDP is not very communicative (i.e., the distribution $p(\cdot|x,a)$ concentrates only on a few states). The rewards are random $r(x,a)=\bar{r}(x,a)(1+\epsilon)$ where $\epsilon\sim\mathcal{N}(0,0.2^2)$ and mean rewards $\bar{r}(x,a)$ are drawn from $\text{Uniform}(0,1)$ and fixed for the problem.

\subsection{Deep RL algorithms}

\paragraph{Proximal policy optimization (PPO).} PPO \citep{schulman2017proximal} implements a stochastic actor $\pi_\theta(a|x)$ as a Gaussian distribution $a\sim \mathcal{N}(\mu_\theta(x),\sigma^2\mathbb{I})$ with state-conditional mean $\mu_\theta(x)$ and a global standard deviation $\sigma^2
\mathbb{I}$; and a value function $V_\phi(x)$. The behavior policy $\mu$ is the previous policy iterate. The policy is updated as $\hat{A}_\gamma^\mu(x,a) \nabla_\theta \text{clip}(\frac{\pi_\theta(a|x)}{\mu(a|x)},1-\epsilon,1+\epsilon)$ with $\epsilon=0.2$\footnote{The exact PPO update is more complicated than this. Refer to \citep{schulman2017proximal} for the exact formula.}. The advantages $\hat{A}_\gamma^\mu(x,a)$ estimated using generalized advantage estimation (GAE, \citep{schulman2015high}) with $\gamma=0.99,\lambda=0.95$. Value functions are trained by minimizing  $(V_\phi(x) - R(x))^2$ with returns $R(x) = V_{\phi^\prime}(x) + \hat{A}_\gamma^\mu(x,a)$ with $\phi^\prime$ being a prior parameter. Both parameters $\theta,\phi$ are trained with the Adam optimizer \citep{kingma2014adam}
with learning rate $\alpha=3\cdot 10^{-4}$.
We adopt other default hyper-parameters in \citep{dhariwal2017openai}, for details, please refer to the code base.

\paragraph{Trust region policy optimization (TRPO).} TRPO \citep{schulman2015high} implements the same actor-critic pipeline as PPO, the difference is in the updates. Instead of enforcing a \emph{soft} clipping constraint, TRPO enforces a strict KL-divegergence constraint $\mathbb{E}_{x\sim\mu}\left\lbrack\text{KL}(\pi_\theta(\cdot|x),\mu(\cdot|x)\right\rbrack \leq \epsilon$ with $\epsilon=0.01$. The policy gradient is computed as $\hat{A}_\gamma^\mu(x,a)\nabla_\theta \log \pi_\theta(a|x)$, and then the final update is constructed by approximately solving a constrained optimization problem, see \citep{schulman2015trust} for details. The scale of the final update is found through a line search, to ensure that the KL-divergence constraint is satisfied.
The implementations are based in \citep{SpinningUp2018}.

\subsection{Deep RL architecture}

Across all algorithms, the policy $\pi_\theta(a|x)=\mathcal{N}(\mu_\theta(x),\sigma^2\mathbb{I})$ has a parameterized mean $\mu_\theta(x)$ and a single standard deviation $\sigma^2$. The mean $\mu_\theta(x)$ is a 2-layer neural network with hidden units $h=64$, and $f(x)=\text{tanh}(x)$ activation functions. The output layer does not have any activation functions; The value function $V_\phi(x)$ is a 2-layer neural network with hidden units $h=64$ and $f(x)=\text{tanh}(x)$ as activation functions. The output layer does not have any activation functions.

\subsection{Additional deep RL experiment results}

\subsubsection{Taylor expansion Q-function estimation: ablation study on $\eta$}

Recall that throughout the experiments, we choose $K=1$ and construct the new Q-function estimator as a mixture of the default estimator and Taylor expansion Q-function estimator. In particular, the final Q-function estimator is
\begin{align*}
    \hat{Q}(x,a) = (1-\eta) \hat{Q}_\gamma^\pi(x,a) + \eta \hat{Q}_{K,\gamma,\gamma'}^\pi(x,a).
\end{align*}

We choose $\eta\in[0,1]$ such that it balances the numerical scales of the two combining estimators. In our implementation, we find that the algorithm performs more stably when $\eta$ is small in the absolute scale. In Figure~\ref{fig:ablation}(a)-(b), we show the ablation study on the effect of $\eta$, where we vary $\eta\in[0.01,0.03]$. The y-axis shows the normalized performance against PPO baselines (which is equivalent to $\eta=0$), such that the PPO baseline achieves a normalized performance of $1$. 

Overall, we see on different tasks, $\eta$ impacts the performance differently. For example: on HalfCheetah(B), better performance is achieved with larger values of $\eta$, this is consistent with the observation that PPO with $\gamma=0.999$ also achieves better performance; on Ant(B), however, as $\eta$ increases from zero, the performance increases marginally before degrading. In Figure~\ref{fig:ablation}, we show the median and mean performance across all tasks. Note that in general, the average performance increases as $\eta$ increases from zero, but later starts to decay a bit. When accounting for the effect of performance variance across all tasks, we chose $\eta=0.01$ as the fixed hyper-parameter throughout experiments in the main paper.

\paragraph{Further details on computing $\hat{Q}_{K,\gamma,\gamma'}^\pi(x,a)$.} Below we assume $K=1$. In Algorithm 4, we showed we can construct unbiased estimates of $Q_{K,\gamma,\gamma'}^\pi(x,a)$ using $\hat{Q}_\gamma^\pi(x,a)$ as building blocks. With a random time $\tau\sim\text{Geometric}(1-\gamma)$, the estimator takes the following form
\begin{align*}
    \hat{Q}_{K,\gamma,\gamma'}(x_t,a_t)=\hat{Q}_\gamma^\pi(x_t,a_t)  + \frac{\gamma'-\gamma}{1-\gamma} Q_\gamma^\pi(x_{t+\tau},a_{t+\tau}).
\end{align*}

However, since the estimator is based on a single random time, it can have high variance. To reduce variance, we propose the following procedure: let $(x_t,a_t)$ be the target state-action pair, we can compute the estimate as 
\begin{align*}
\hat{Q}_{K,\gamma,\gamma'}(x_t,a_t)=\hat{Q}_\gamma^\pi(x_t,a_t)  + \frac{\gamma'-\gamma}{1-\gamma} \sum_{s=1}^H \frac{\gamma^s}{\sum_{s'=1}^H \gamma^{s'}}Q_\gamma^\pi(x_{t+s},a_{t+s}).
\end{align*}
 When $H=\infty$, the above estimator corresponds to an estimator which marginalizes over the random time. This should achieve variance reduction compared to the random time based estimate in Algorithm 4. However, then the estimate requires computing cumulative sums
over an infinite horizon (or in general a horizon of $T$), which might be computationally expensive.
To mitigate this, we propose to truncate the above summation up to $H=10$ steps. This choice of $H$ aims to achieve a trade-off between computation efficiency and variance. Note that this estimator was previously introduced in \citep{tang2020taylor} for off-policy learning.

\subsubsection{Taylor expansion update weighting: ablation on $K$}

In Figure~\ref{fig:ablation}(c)-(d), we carry out ablation study on the effect of $K$ for the update weighting. Recall that $K$ interpolates two extremes: when $K=0$, it recovers the vanilla PG \citep{sutton2000policy} while when $K=\infty$, it recovers the deep RL heuristic update. We expect an intermediate value of $K$ to achieve some trade-off between bias and variance of the overall update.

In Figure~\ref{fig:ablation}(c), we see the effect on individual environments. The effect is case dependent. For HalfCheetah(G), larger $K$ improves the performance; however, for Walker(G), the improvement is less prominent over a large range of $K$. When aggregating the performance metric in Figure~\ref{fig:ablation}(d), we see that intermediate values of $K$ indeed peak in performance. We see that on average, both $K=10$ and $K=100$ achieve locally optimal mean performance, while $K=10$ also achieves the locally optimal median performance.

\paragraph{Note on how the practical updates impact the effect of $K$.} Based on our theoretical analysis, when $K=0$ the update should recover the vanilla PG \citep{sutton2000policy}, which is generally considered too conservative for the undiscounted objective in Eqn~\eqref{eq:rl-obj}. However, in practice, as shown in Figure~\ref{fig:ablation}(d), the algorithm does not severely underperform even when $K=0$. We speculate that this is because practical implementations of PG updates use batches of data instead of the full trajectories. This means that the relative weights $w(t)$ of the local gradients $\hat{Q}_t \nabla_\theta \log \pi_\theta(a_t|x_t)$ are effectively self-normalized: $\tilde{w}(t)\leftarrow \frac{w(t)}{\sum w(t')}$ where the summation is over the time steps in a sampled mini-batch. The self-normalized weights $\tilde{w}(t)$ are increased in the absolute scale relative to $w(t)$ and partly offset the effect of an initially aggressive discount $w(t)=\gamma^t$.

\subsubsection{Comparison to results in \citep{romoff2019separating}}

Recently, \citet{romoff2019separating} derived a recursive relations between differences value functions defined with different discount factors. This was shown in Lemma~\ref{lemma:value-mixture}. Given a sequence of discount factors $\gamma_1<\gamma_2<\ldots<\gamma_N<\gamma'$, they derived a value function estimator to $V_{\gamma'}^\pi(x)$ based on recursive bootstraps of value function differences $V_{\gamma_i}^\pi(x) - V_{\gamma_{i-1}}^\pi(x')$. Because they aim at recovering the exact value functions, this estimator could be interpreted as similar to Taylor expansions but with $K=\infty$. 

Different from their motives, we focus on the trade-off achieved by intermediate values of $K$. We argued that by using $K=0$, the estimate might be too conservative; however, using $K=\infty$ might be challenging due to the variance induced in the recursive bootstrapping procedure. Though it is not straightforward to theoretically show, we conjecture that using the Taylor expansion Q-function estimator with $K=\infty$ is as difficult as directly estimating $V_{\gamma'}^\pi(x)$.

\begin{figure}
    \centering
 \includegraphics[width=.5\linewidth]{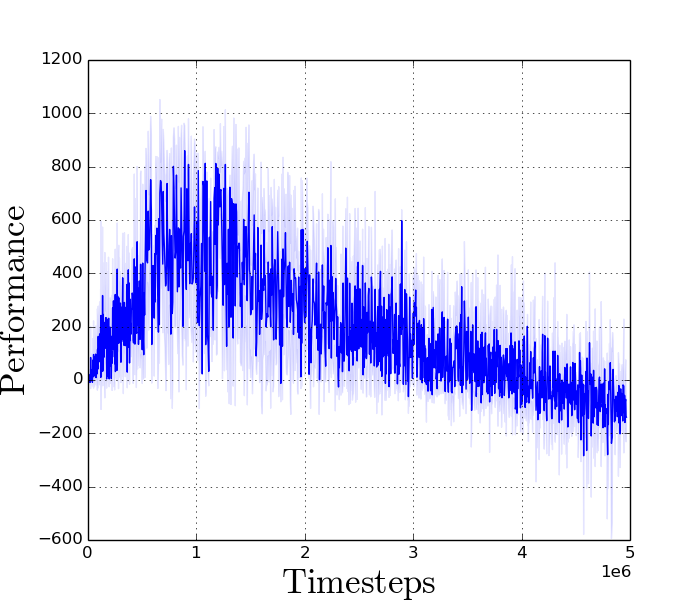}
   \caption{Learning curves generated by running the open source implementation of \citep{romoff2019separating} on Walker2d(G), averaged across $5$ runs. There is little progress of learning for the algorithm on other benchmark tasks.}
   \label{fig:separating}
\end{figure}

\paragraph{Empirical comparison.}
The base algorithm of \citep{romoff2019separating} is  PPO\citep{schulman2017proximal}. Their algorithm uses the recursive bootstraps to estimate Q-functions and advantage functions. The new estimate is used as a direct plug-in replacement to $\hat{Q}_\gamma^\pi(x,a)$ and $\hat{A}_\gamma^\pi(x,a)$ adopted in the PPO algorithm. We run experiments with the open source implementation of \citep{romoff2019separating} from the original authors\footnote{See \url{https://github.com/facebookresearch/td-delta}.}. We evaluate the algorithm's performance over continuous control benchmark tasks.  We applied the default configurations from the code base with minimum changes to run on continuous problems (note that \citep{romoff2019separating} focused on a few discrete control problems). Overall, we find that the algorithm does not learn stably (see Figure~\ref{fig:separating}).

\begin{figure}[t]
    \centering
    \subfigure[Ablation on $\eta$ (individual)]{\includegraphics[keepaspectratio,width=.22\textwidth]{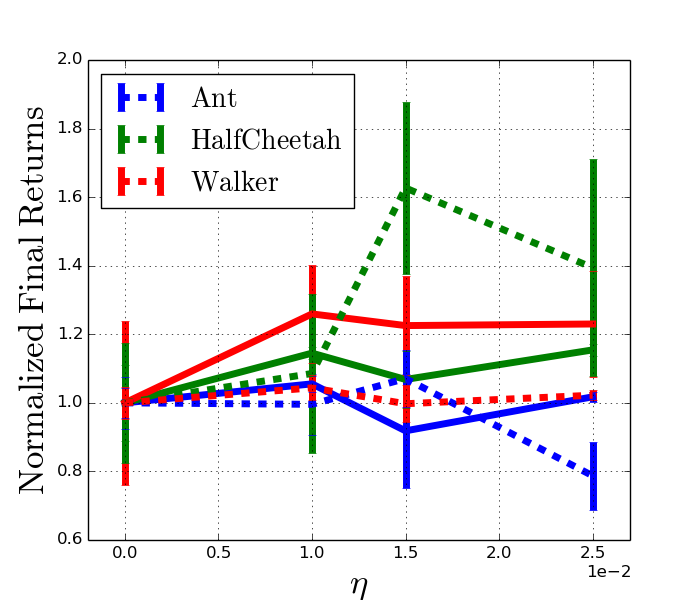}}
    \subfigure[Ablation on $\eta$ (average)]{\includegraphics[keepaspectratio,width=.22\textwidth]{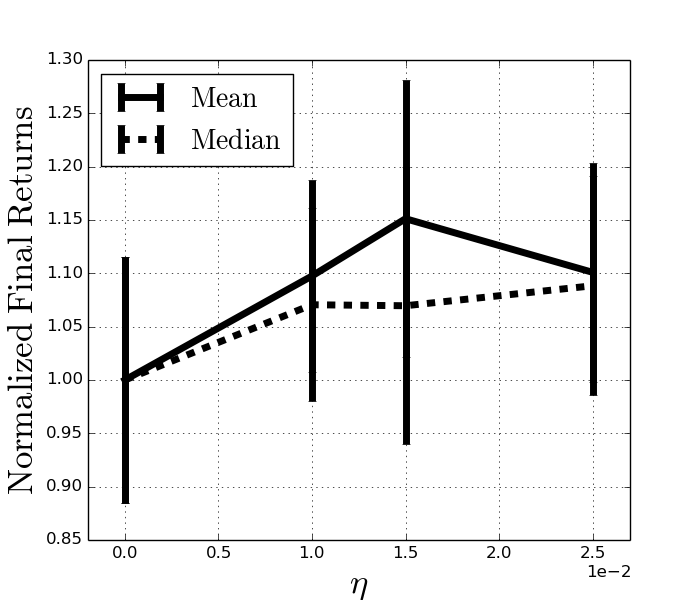}}
    \subfigure[Ablation on $K$ (individual)]{\includegraphics[keepaspectratio,width=.22\textwidth]{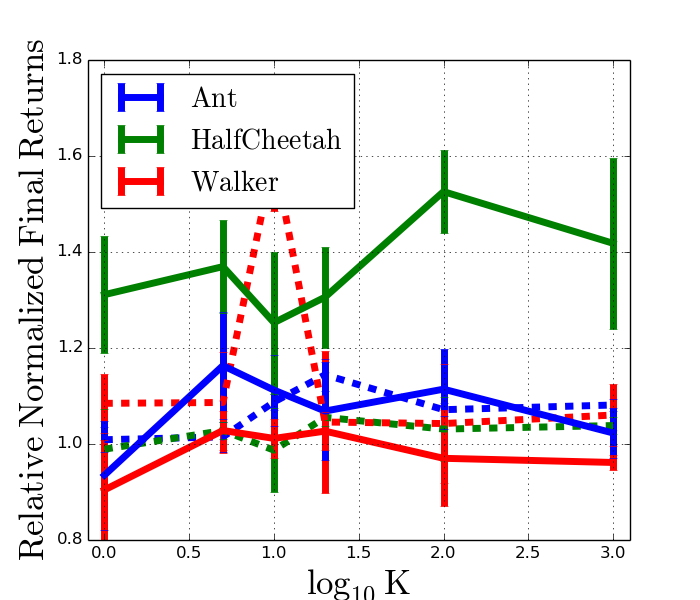}}
    \subfigure[Ablation on $K$ (average)]{\includegraphics[keepaspectratio,width=.22\textwidth]{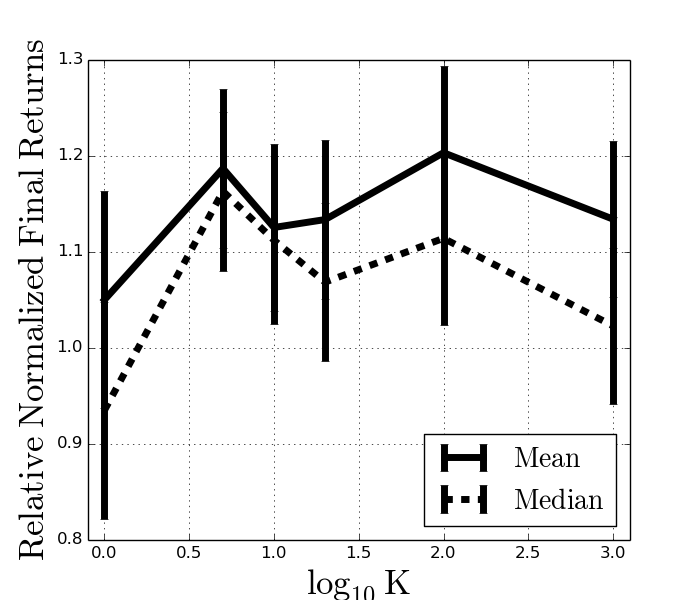}}
    \caption{Ablation study of hyper-parameters. We study two hyper-parameters: (a) $\eta$ (b) $K$. In both cases, we calculate the task-dependent normalized final returns after training for $10^6$ steps. See Appendix~\ref{appendix:exp} for how such normalized returns are computed. In (a),  normalized returns are computed with respect to $\eta=0$ (i.e, the PPO baseline), such that when $\eta=0$, the normalized returns are ones; in (b), normalized returns are computed with respect to the default PPO baseline, such that values of ones imply that the baseline performs the same as the default PPO baseline. Dashed curves (bullet tasks) and solid curves (gym tasks) are both mean scores averaged over $5$ seeds.}
    \label{fig:ablation}
\end{figure}

\section{Extensions of update weighting techniques to off-policy algorithms}

Below, we show that techniques developed in this paper could be extended to off-policy learning algorithms. We provide both details in theoretical derivations, algorithms, as well as experimental results.

\subsection{Off-policy actor-critic algorithms}

Off-policy actor-critics \citep{mnih2015human,lillicrap2015continuous} maintain a deterministic policy $\pi_\theta(x)$ and a Q-function critic $Q_\phi(x,a)$. The agent takes exploratory actions under the environment, and saves data $(x_t,a_t,r_t)$ into a common replay buffer $\mathcal{D}$. At training time, the algorithm samples data from the replay to update parameters. The policy is updated via the deterministic policy gradient \citep{silver2014deterministic}, $\theta\leftarrow\theta + \alpha\nabla_\theta \mathbb{E}_\mu\left\lbrack Q_\phi(x,\pi_\theta(x)) \right\rbrack$, where $\mu$ is implicitly defined by the past behavior policy.

\paragraph{Deep deterministic policy gradient (DDPG).} DDPG \citep{lillicrap2015continuous} maintains a deterministic policy network $\pi_\theta(a|x)\equiv \pi_\theta(x)$ and a Q-function critic $Q_\phi(x,a)$. The algorithm explores by executing a perturbed policy $a = \epsilon + \pi_\theta(x)$ where $\epsilon\sim\mathcal{N}(0,\sigma^2)$ for $\sigma=0.1$, and then saves the data $(x,a,r,x^\prime)$ into a replay buffer $\mathcal{D}$. At training time, the behavior data is sampled uniformly from the replay buffer $(x_i,a_i,r_i,x_i^\prime)_{i=0}^{B-1}\sim\mathcal{U}(\mathcal{D})$ with $B=100$. The critic is updated via TD($0$), by minimizing: $\frac{1}{B}\sum_{i=0}^{B-1}(Q_\phi(x_i,a_i)-Q_{\text{target}}(x_i,a_i))^2$ where $Q_\text{target}(x_i,a_i) = r_i + \gamma Q_{\phi^\prime}(x_i^\prime,\pi_{\theta^\prime}(x_i^\prime))$, where $\theta^\prime,\phi^\prime$ are delayed versions of $\theta,\phi$ respectively \citep{mnih2015human}. The policy is updated by maximizing $\frac{1}{B}\sum_{i=0}^{B-1} Q_\phi(x_i,\pi_\theta(x_i))$ with respect to $\theta$. Both parameters $\theta,\phi$ are trained with the Adam optimizer \citep{kingma2014adam} with learning rate $\alpha=10^{-4}$. We adopt other default hyper-parameters in \citep{SpinningUp2018}, for details, please refer to the code base.

\paragraph{Twin-delayed DDPG (TD3).} TD3 \citep{fujimoto2018addressing} adopts the same training pipeline and architectures as DDPG. TD3 also adopts two critic networks $Q_{\phi_1}(x,a),Q_{\phi_2}(x,a)$ with parameters $\phi_1,\phi_2$, in order to minimize the over-estimation bias \citep{hasselt2010double,van2016deep}.

\paragraph{Soft actor-critic.} SAC \citep{haarnoja2018soft} adopts a similar training pipeline and architectures as TD3. A major conceptual difference is that SAC is based on the maximum-entropy formulation of RL \citep{ziebart2008maximum,fox2015taming}. The Q-function is augmented by entropy regularization bonus and the policy is optimized such that it does not collapse to a deterministic policy.

\subsection{Architecture}

All algorithms share the same architecture. The policy network $\pi_\theta(x)$ takes as input the state $x$, and is a 2-layer neural network with hidden units $h=256$ and $f(x) =\text{relu}(x)$ activation functions. The output is squashed by $f(x)=\text{tanh}(x)$ to comply with the action space boundaries; The critic $Q_\phi(x,a)$ takes a concatenated vector $[x,a]$ as inputs, is 2-layer neural network with hidden units $h=256$ and $f(x) =\text{relu}(x)$ activation functions. The output does not have any activation functions.

For stochastic policies, the policy network parameterizes a Gaussian also parameterizes a log standard deviation vector $\log \sigma(x)$, which is a neural network with the same architecture above. The stochastic output is a reparameterized function $a=\pi_\theta(x)+\exp(\log\sigma(x))\cdot \epsilon$ where the noise $\epsilon\sim\mathcal{N}(0,1)$. Finally, the action output is squashed by $\text{tanh}(x)$ to comply with the action boundary \citep{haarnoja2018soft}.

\subsection{Algorithm details for update weighting}
To derive an update based on update weighting, we start with the undiscounted on-policy objective $V_{\gamma'}(x)=\mathbb{E}_{x'\sim \rho_{x,\gamma,\gamma'}^{\pi_\theta}}\left[V_\gamma^{\pi_\theta}(x')\right]$. Given behavior data generated under $\mu$, we abuse the notation and also use $\mu$ to denote the state distribution under $\mu$ (usually implicitly defined by sampling from a replay buffer $\mathcal{D}$). By rewriting the objective with importance sampling (IS), 
\begin{align}
    V_{\gamma'}^{\pi_\theta}(x) = \mathbb{E}_{x^\prime\sim\rho_{x,\gamma,\gamma^\prime}^{\pi_\theta}}\left\lbrack V_\gamma^{\pi_\theta}(x^\prime)\right\rbrack =  \mathbb{E}_{x'\sim\mu}\left\lbrack \frac{\rho_{x,\gamma,\gamma'}^{\pi_\theta}(x')}{\mu(x')}V_\gamma^{\pi_\theta}(x^\prime)\right\rbrack,
\end{align}
we derive an off-policy learning objective. By dropping a certain terms (see \citep{degris2012off} for details about the justifications for dropping such terms), we can derive the IS-based gradient update
\begin{align*}
    \mathbb{E}_{x^\prime\sim\mu}\left\lbrack \frac{\rho_{x,\gamma,\gamma^\prime}^{\pi_\theta}(x^\prime)}{\mu(x^\prime)} \nabla_\theta V_\gamma^{\pi_\theta}(x^\prime)\right\rbrack \approx  \mathbb{E}_{x^\prime\sim\mu}\left\lbrack \frac{\rho_{x,\gamma,\gamma^\prime}^{\pi_\theta}(x^\prime)}{\mu(x^\prime)}  \nabla_\theta Q_\phi\left(x',\pi_\theta(x')\right)\right\rbrack
\end{align*}
To render the update feasible, we need to estimate the ratio $\frac{\rho_{x,\gamma,\gamma^\prime}^{\pi_\theta}(x^\prime)}{\mu(x^\prime)} $. Inspired by \citep{sinha2020experience}, we propose to maintain a \emph{fast replay buffer} $\mathcal{D}_f$ which contains the most recent sampled data (which implicitly defines $\rho_{x,\gamma,\gamma'}^{\pi_\theta}$), then the estimator $w_\psi$ is trained to estimate the density ratio between $\mathcal{D}$ (which implicitly defines $\mu$) and $\mathcal{D}_f$. See Appendix \ref{appendix:exp} for further details. The full off-policy actor-critic algorithm is summarized in Algorithm~\ref{algo:off}. In practice, we implement a undiscounted uniform distribution instead of $\rho_{x,\gamma,\gamma^\prime}^\pi(x')$ with $\gamma'=1$. The main motivation is that this distribution is much easier to specify as it corresponds to sampling from the replay buffer uniformly without discounts, as explained below.

As an important observation for practical implementations, note that
\begin{align*}
    \rho_{x,\gamma,\gamma'}^\pi(x') = \frac{\gamma}{\gamma'} \mathbb{I}[x_0=x'] + (\gamma'-\gamma) \mathbb{E}_\pi\left[\sum_{t\geq 1} (\gamma')^{t-1} \mathbb{I}[x_t=x'] \; \middle| \; x_0=x \right]
\end{align*}
when setting $\gamma'=1$, we see that the second term of the distribution is proportional to  $\mathbb{E}_\pi\left[ \sum_{t\geq 1} \mathbb{I}[x_t=x'] \; \middle| \; x_0=x \right] $, which corresponds to a uniform distribution over states on sampled trajectories, \emph{without} discounting. This will make implementations much simpler. We will see that this could also lead to performance gains. We leave Taylor expansion based extension of this method for future work.

\paragraph{Details on training the density estimator $w_\psi(x)$.} The density estimator $w_\psi(x)$ is parameterized with exactly the same architecture as the policy network $\pi_\theta(x)$, except that its output activation is replaced by $\log(1+\exp(x))$ to ensure that $w_\psi(x)> 0$. The off-policy actor-critic algorithm maintains an original buffer $\mathcal{D}$ of size $|\mathcal{D}|=10^6$; in addition, we maintain a fast replay buffer $\mathcal{D}_f$ with $|\mathcal{D}_f|=10^4$, which is used for saving the most recently generated data points. For ease of analysis, assume that the data sampled from $\mathcal{D}_f$ come from $\pi_\theta$, while the data sampled from $\mathcal{D}$ come from $\mu$. 

To learn the ratio $\frac{\rho_{x,\gamma,\gamma'}^{\pi_\theta}(x')}{\mu(x')}$, we adopt a simple discriminative loss function as follows
\begin{align*}
    L(\psi) = -\mathbb{E}_{x'\sim \rho_{x,\gamma,\gamma'}^{\pi_\theta}}\left\lbrack \log \frac{w_\psi(x')}{1+w_\psi(x')} \right\rbrack
    - \mathbb{E}_{x'\sim \mu}\left\lbrack \log \frac{1}{1+w_\psi(x')} \right\rbrack
    \approx -\mathbb{E}_{x\sim \mathcal{D}_f}\left\lbrack \log \frac{w_\psi(x')}{1+w_\psi(x')} \right\rbrack
    - \mathbb{E}_{x\sim \mathcal{D}}\left\lbrack \log \frac{1}{1+w_\psi(x')} \right\rbrack.
\end{align*}
The optimal solution to $\psi^\ast = \arg\min_\psi L(\psi)$ is $w_{\psi^\ast}(x') = \frac{\rho_{x,\gamma,\gamma'}^{\pi_\theta}(x')}{\mu(x')}$(assuming enough expressiveness). Then, the density estimator is used for weighting the policy update: when sampling a batch of $B$ data from the buffer, the weight $w_\psi(x_i),1\leq i\leq B$ is computed for each data point $x_i$. Then the weights are normalized across batch $\tilde{w}_i = \frac{w_\psi(x_i)^\tau}{\sum_{j=1}^B w_\psi(x_j)^\tau}$ where the inverse temperature is $\tau=0.1$. Then $\tilde{w}_i$ is used for weighting the such that the policy is updated as $\theta\leftarrow\theta + \alpha \frac{1}{B}\sum_{i=1}^B \tilde{w}_i \nabla_\theta Q_\phi(x_i,\pi_\theta(x_i))$.

\begin{algorithm}[h]
\label{algo:off}
\begin{algorithmic}
\REQUIRE  policy $\pi_\theta(x)$, Q-function critic $Q_\phi(x,a)$, density estimator $w_\psi(x)$ and learning rate $\alpha \geq 0$ \\
\WHILE{not converged}
\STATE 1. Collect data $(x_t,a_t,r_t)\sim\mu$ and save to the buffer $\mathcal{D}$ and the fast buffer $\mathcal{D}_f$
\STATE 2. Estimate the density by the discriminative loss between $\mathcal{D},\mathcal{D}_f$, such that $w_\psi(x^\prime)\approx \rho_{x,\gamma,\gamma^\prime}^{\pi_\theta}(x^\prime)/\mu(x^\prime)$, where $x$ is the initial state of the MDP.
\STATE 3. Sample data from $(x_i,a_i,r_i)_{i=1}^B \sim \mathcal{D}$.
\STATE 3(a). Update the Q-function critic $Q_\phi(x,a)$ via TD-learning, such that $Q_\phi(x,a)\approx Q_\gamma^{\pi_\theta}(x,a)$.
\STATE 3(b). Update the policy parameter with the gradient $\theta\leftarrow\theta + \alpha \sum_{i=1}^B w_\psi(x_i) \nabla_\theta Q_\phi(x_i,\pi_\theta(x_i))$.
\ENDWHILE
\caption{Update weighting Off-policy actor-critic}
\end{algorithmic}
\end{algorithm}

We carry out the update in Algorithm 2, where the density estimator $w_\psi(x)$ is trained based on a discriminative loss between $\mathcal{D}$ and $\mathcal{D}_f$. For any given batch of data $\{x_i\}_{i=1}^B$, we normalize the prediction $\tilde{w}_i=w_\psi(x_i)^{\tau}/\sum_{j=1}^B w_\psi(x_j)^\tau$ with hyper-parameter $\tau=0.1$ as similarly implemented in \citep{sinha2020experience}. The temperature annealing moves $\tilde{w}_i$ closer to a uniform distribution and tends to stabilize the algorithm. See Appendix~\ref{appendix:exp} for further details.

\paragraph{Discussion on relations to other algorithms.} Previous work focuses on re-weighting transitions to stabilize the training of critics. For example, prioritized replay \citep{schaul2015prioritized} prioritizes samples with high Bellman errors. Instead, Algorithm 2 reweighs samples to speed up the training of the policy. Our observation above also implies that when sampling from $\mathcal{D},\mathcal{D}_f$ for training the estimates $w_\psi\approx \frac{\rho_{x,\gamma,\gamma'}^{\pi_\theta}(x')}{\mu(x')}$, it is not necessary to discount the transitions. This is in clear contrast to prior work, such as \citep{sinha2020experience}, where they propose to train $w_\psi(x')\approx d_{x,\gamma}^{\pi_\theta}(x') / d_{x,\gamma}^\mu(x')$ , which is the fully discounted  visitation distribution under $\gamma$ based on the derivation of optimizing a discounted objective $V_\gamma^{\pi_\theta}(x)$.

\begin{figure}
    \centering
    \subfigure[HalfCheetah(G)]{\includegraphics[keepaspectratio,width=.22\textwidth]{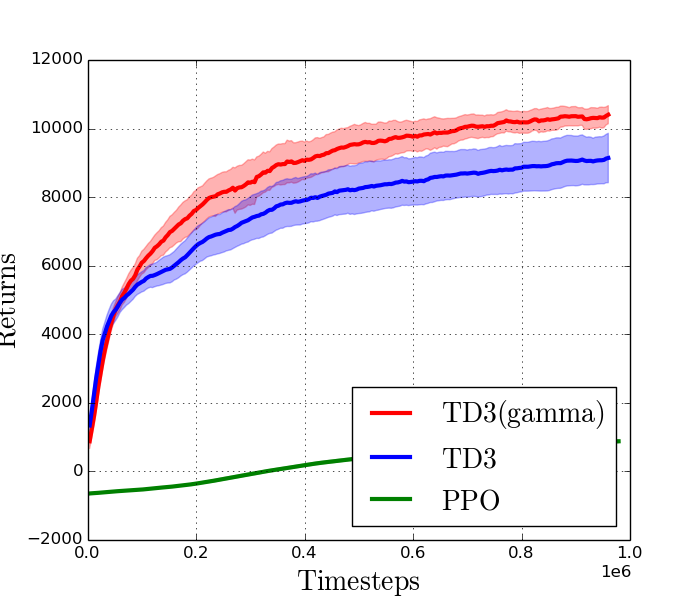}}
    \subfigure[Ant(G)]{\includegraphics[keepaspectratio,width=.22\textwidth]{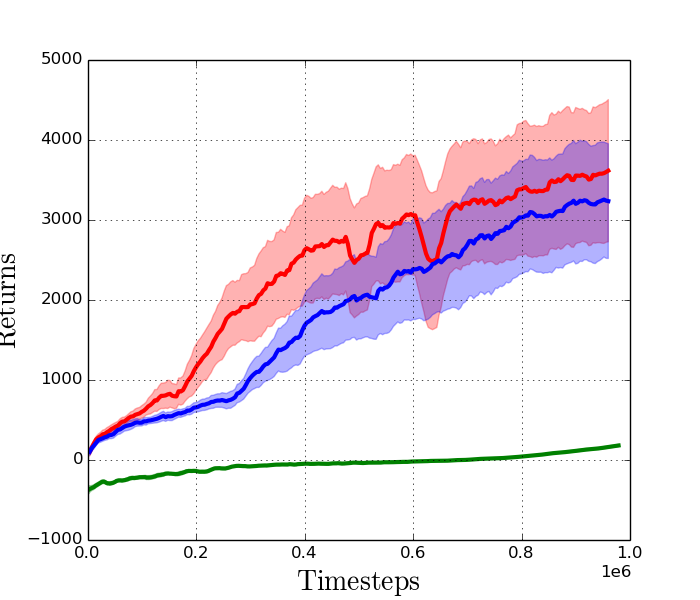}}
    \subfigure[Walker2d(G)]{\includegraphics[keepaspectratio,width=.22\textwidth]{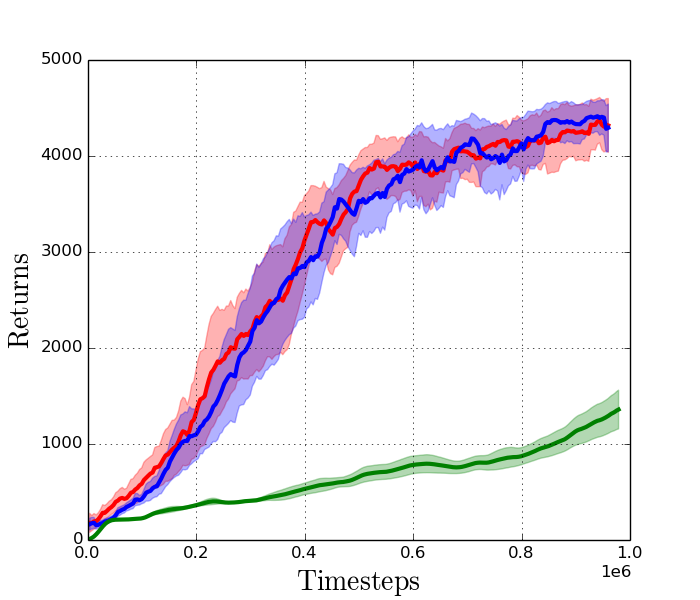}}
    \subfigure[Hopper(G)]{\includegraphics[keepaspectratio,width=.22\textwidth]{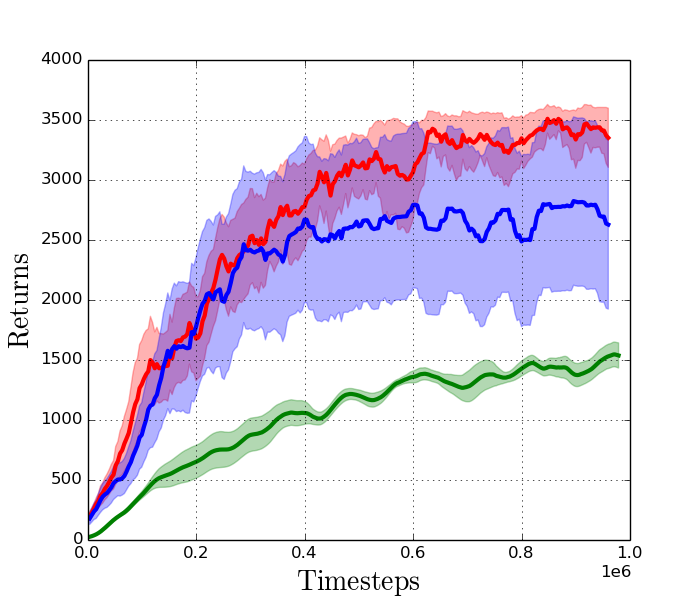}}
    \caption{Evaluation of near off-policy actor-critic algorithms over continuous control domains. Each curve corresponds to a baseline algorithm averaged over $5$ random seeds. TD3($\gamma$) (red curve) consistently outperforms or performs similarly as other baselines.}
    \label{fig:td3}
\end{figure}

\begin{figure}
    \centering
    \subfigure[HalfCheetah(G)]{\includegraphics[keepaspectratio,width=.22\textwidth]{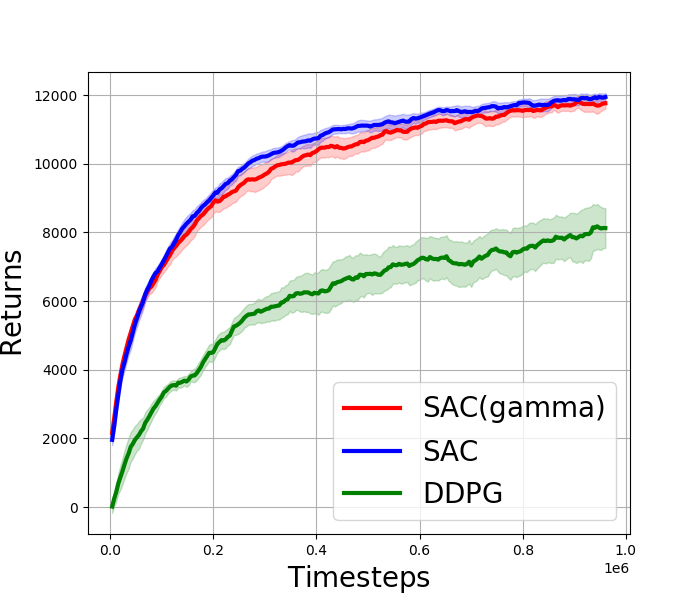}}
    \subfigure[Ant(G)]{\includegraphics[keepaspectratio,width=.22\textwidth]{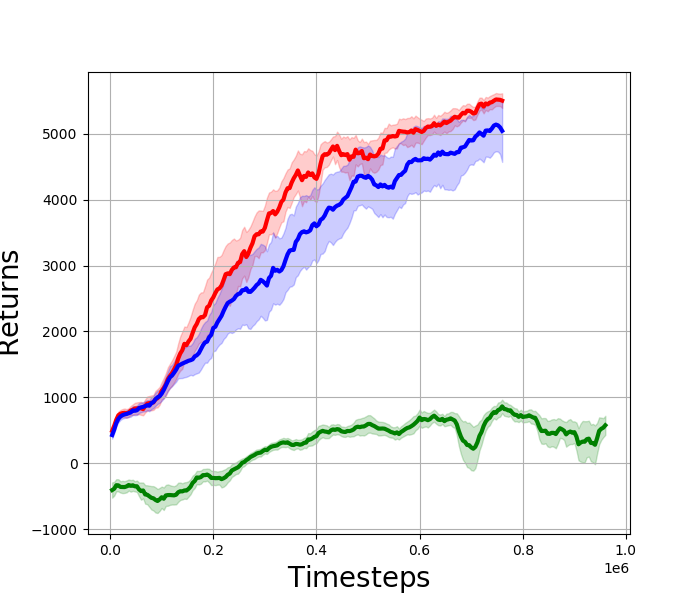}}
    \subfigure[Walker2d(G)]{\includegraphics[keepaspectratio,width=.22\textwidth]{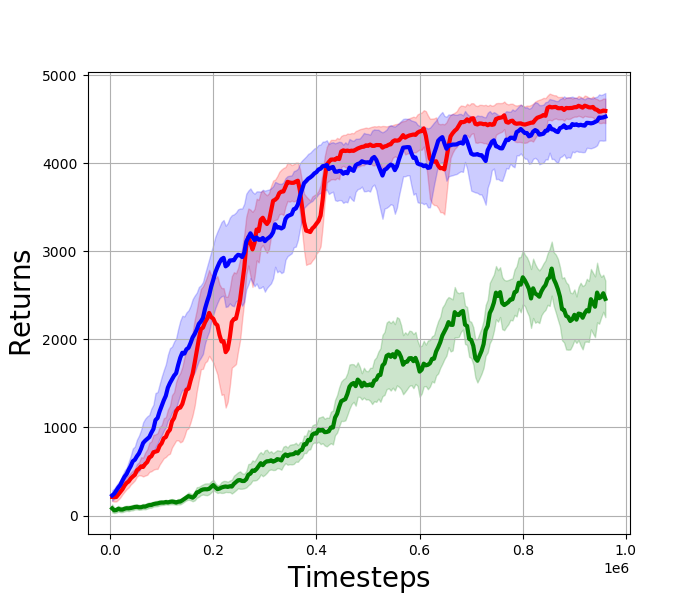}}
    \subfigure[Hopper(G)]{\includegraphics[keepaspectratio,width=.22\textwidth]{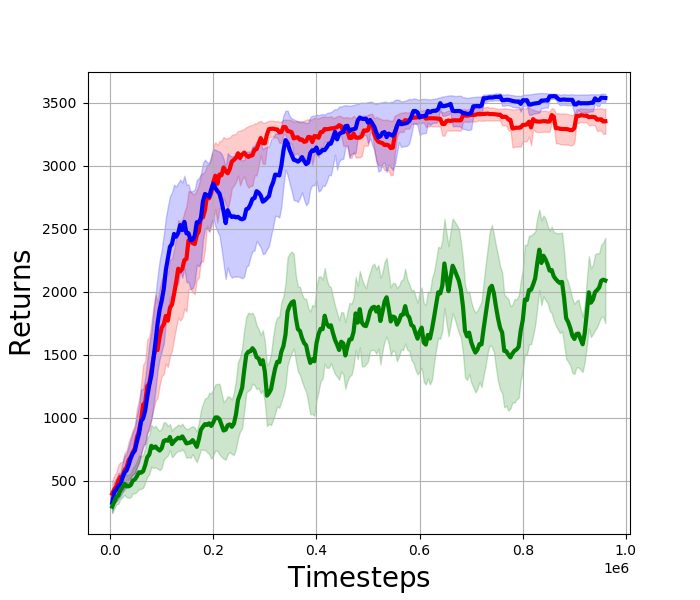}}
    \caption{Evaluation of near off-policy actor-critic algorithms over continuous control domains. Each curve corresponds to a baseline algorithm averaged over $5$ random seeds. SAC($\gamma$) (red curve) consistently outperforms or performs similarly as other baselines.}
    \label{fig:sac}
\end{figure}

\paragraph{Results.} We build the algorithmic improvements based on TD3 \citep{fujimoto2018addressing} and SAC \citep{haarnoja2018soft}, and name the correspinding algorithms TD3($\gamma$) and SAC($\gamma$) respectively. We compare with TD3, SAC, and DDPG \citep{lillicrap2015continuous}, all of which are off-policy algorithms.

We first compare TD3($\gamma$) with TD3 in Figure~\ref{fig:td3}. To highlight the default sample efficiency of off-policy methods, we include PPO as a baseline as well. Across all four presented tasks, we see that TD($\gamma$) performs similarly or marginally outperforms the TD3 baseline. To make concrete the comparison between final performance, we report the final score $\text{mean}\pm0.5\text{std}$ of each algorithm in Table~\ref{table:td3}. As a default baseline, we also show the results of DDPG reported in \citep{SpinningUp2018}. Overall, TD3($\gamma$) provides a modest yet consistent boost over baseline TD3.

Then we compare SAC($\gamma$) with SAC in Figure~\ref{fig:sac} and Table~\ref{table:td3}. We see that SAC($\gamma$) provides marginal performance gains over Walker2d and Ant, while it is slightly overperformed by baseline SAC for HalfCheetah and Hopper. We speculate that this is partly because the hyper-parameters of baseline SAC are well tuned on HalfCheetah, and it is difficult to achieve further significant gains without exhaustive hyper-parameter search. Overall, SAC($\gamma$) is competitive compared to SAC.

\begin{table}
    \vskip 0.1in
    \begin{center}
    \footnotesize
    \begin{sc}
    \begin{tabular}{c|c|c|c}\toprule[1.5pt]
        \bf Tasks & \bf TD3($\gamma$) & \bf TD3 & \bf DDPG-v1 \\\midrule
Ant(G) & $\mathbf{3601 \pm 879}$ & $\mathbf{3269 \pm 686}$ & $\approx 1000$    \\
HalfCheetah(G) & $\mathbf{10350 \pm 279}$ & $9156 \pm 718$ & $\approx 8500$    \\
Walker2D(G) & $\mathbf{4090 \pm 440}$ & $\mathbf{4233 \pm 314}$ & $\approx 2000$    \\
Hopper(G) & $\mathbf{3340 \pm 262}$ & $2626 \pm 677$ & $\approx 1800$   \\  \toprule[1.5pt]
\bf Tasks & \bf SAC($\gamma$) & \bf SAC & \bf DDPG-v2  \\\midrule
Ant(G) & $\mathbf{5572 \pm 115}$ & $4886 \pm 530$ & $706 \pm 123$    \\
HalfCheetah(G) & $11774 \pm 96$ & $\mathbf{12059 \pm 91}$ & $7957 \pm 527$    \\
Walker2D(G) & $\mathbf{4626 \pm 165}$ & $\mathbf{4522 \pm 269}$ & $2261 \pm 147$    \\
Hopper(G) & $3384 \pm 81$ & $\mathbf{3557 \pm 20}$  & $2024 \pm 297$    \\
 \bottomrule[1.46pt]
    \end{tabular} \par
    \end{sc}
    \end{center}
    \caption{Final performance of baseline algorithms over benchmark tasks. The final performance is computed as the mean scores over the last $10$ iterations of each algorithm, averaged over $5$ seeds. When compared with TD3, the performance of DDPG-v1 is taken from \citep{SpinningUp2018}; when compared with SAC, the performance is based on re-runs of the DDPG-v2 baselines with \citep{SpinningUp2018}. For each task, the best algorithms are highlighted in bold fonts (potentially with ties).}
    \label{table:td3}
\end{table}

\end{appendix}

\end{document}